\documentclass{article}
\usepackage{iclr2027_conference,times}

\usepackage{amsmath,amsfonts,bm}

\def\eqref#1{equation~\ref{#1}}
\def\1{\bm{1}}

\DeclareMathAlphabet{\mathsfit}{\encodingdefault}{\sfdefault}{m}{sl}
\SetMathAlphabet{\mathsfit}{bold}{\encodingdefault}{\sfdefault}{bx}{n}

\usepackage{url}
\usepackage{adjustbox}
\usepackage{tcolorbox}
\usepackage{amsmath}
\usepackage{hyperref}
\usepackage{xfrac}
\usepackage{amssymb}
\usepackage[table]{xcolor}
\usepackage{nicematrix}
\usepackage{booktabs}
\usepackage{longtable}
\usepackage{multirow}
\usepackage{pifont}
\usepackage{wrapfig}
\usepackage{caption}
\usepackage{subcaption}
\usepackage{graphicx}
\iclrfinalcopy % Uncomment for camera-ready version, but NOT for submission.
\usepackage{capt-of}
\usepackage{float}
\usepackage{placeins}
\usepackage{algorithm}
\definecolor{darkgreen}{HTML}{1B5E20}

\newcommand{\cmark}{\ding{51}} % ✓
\newcommand{\xmark}{\ding{55}} % ✗
\newcommand{\best}[1]{\textcolor{red}{\textbf{#1}}}
\newcommand{\second}[1]{\textcolor{blue}{\textbf{#1}}}
\newcommand{\tblhead}{\rule[-3pt]{0pt}{12pt}}
\newcommand{\tblfour}{\rule[-3pt]{0pt}{12.5pt}}

\hypersetup{
  colorlinks=True,
  pdfborder={0 0 0},
  citecolor=blue,    % \cite 的颜色
}
\newtheorem{theorem}{Theorem}
\renewcommand{\bibsection}{%
    \section*{\refname}%
    \phantomsection
    \addcontentsline{toc}{section}{\refname}}

\title{Learning Skills from Historical Action Trajectories: Action Experience Dictionary for World Action Models}
\author{
\makebox[\textwidth][l]{%
\begin{minipage}{\textwidth}
\raggedright
\normalfont
\textbf{Qi Lyu}$^{1,*}$,
\textbf{Jiahua Dong}$^{2,*}$,
\textbf{Hao Shen}$^{3}$,
\textbf{Xudong Wang}$^{1}$,
\textbf{Hongyuan Yu}$^{4}$,
\textbf{Baichen Liu}$^{1,\dagger}$, \\
\textbf{Henghui Ding}$^{5}$,
\textbf{Zhi Han}$^{1}$,
\textbf{Nicu Sebe}$^{6}$,
\textbf{Ivan Laptev}$^{2}$,
\textbf{Fahad Shahbaz Khan}$^{2}$,
\textbf{Salman Khan}$^{2}$\\
$^{1}$Shenyang Institute of Automation, Chinese Academy of Sciences \\
$^{2}$Mohamed bin Zayed University of Artificial Intelligence \\
$^{3}$Anhui University~~
$^{4}$Xiaomi Corporation~~
$^{5}$Fudan University~~
$^{6}$University of Trento \\
$^{*}$Equal contributions~~
$^{\dagger}$Corresponding Author
\end{minipage}%
}
}

\begin{document}

\maketitle
\begin{abstract}
World Action Models (WAMs) couple visual dynamics prediction with action generation, yet they do not explicitly support the reuse of action experience across manipulation tasks. Furthermore, existing WAMs struggle to capture underlying cross-task semantic relationships that could guide target action prediction, as redundant background elements interfere with the extraction of key visual information. To address these challenges, we develop a novel Action Experience Dictionary (AED) that encodes historical physical action trajectories into shared action embeddings to support skill reuse and model cross-task relationships. 
Specifically, we first aggregate historical actions to align with visual observations and retrieve action embeddings from the AED using a pretrained action tokenizer. Subsequently, we visually condition the pooled embeddings through cross-attention and prepend them to noisy action tokens, providing interaction context and action intent for prediction. To model action-related motion and reduce reliance on irrelevant background cues, we introduce a motion-aware transition loss that supervises visual feature change prediction over random temporal intervals. Experiments on simulation benchmarks and in real-world cross-embodiment settings verify the effectiveness of our AED. The project website is available at \url{https://github.com/JiahuaDong/AED}.

\end{abstract}

\section{Introduction}
Rapid advances in large-scale foundation models~\citep{touvron2023llama,qwen3,wan2025,deepseekai2025deepseekv3technicalreport} have spurred interest in transferring visual and linguistic knowledge to physical robots. Vision-Language-Action (VLA) models~\citep{kim2024openvla,black2024pi0,bai2026flash,jia2026action} adapt pretrained vision-language representations to map observations and instructions to actions, but their direct policy formulations often leave environmental dynamics implicit. To exploit dynamics in video data, World Action Models (WAMs)~\citep{an2026feedback,lingbot-va2026,jia2026physics} couple action generation with future visual prediction, providing supervision on how environments evolve. Recent WAM advances span unified video-action architectures and efficient control pipelines.

However, existing WAMs \citep{chen2026lawam,yang2026lila} typically focus on learning task-specific behaviors, leaving the potential of cross-task semantic relationships to guide manipulation skill learning underexplored. In particular, such underlying relationships among manipulation tasks can help robots draw on action experience relevant to the target task, thereby improving manipulation performance.
As illustrated in Fig.~\ref{fig:motivation}(a), a robot learning to place a can into a basket can build on experience in grasping, transporting, and releasing objects acquired from other pick-and-place tasks (\emph{e.g.}, placing a wine bottle on a shelf or opening a drawer and placing a bowl inside). By adapting these behaviors to the basket's position and opening, the robot can learn the target manipulation task more effectively. Similarly, action experience gained from placing a can into a basket can also facilitate the learning of related actions in tabletop pick-and-place tasks. 
This practical example demonstrates that different tasks share reusable action experience despite their distinct goals and visual contexts, as depicted in Fig.~\ref{fig:motivation}(b). Nevertheless, simply retaining historical actions is insufficient to make this knowledge reusable, since similar motions can serve different purposes depending on the objects being manipulated and their spatial relationships. Moreover, redundant background elements may hinder the extraction of action-relevant visual information. 
These challenges lead us to the central question: 
\textit{How can WAMs~\citep{cai2026ahawam} use past action trajectories to model underlying relationships among tasks and facilitate skill learning for the target manipulation task?}

\begin{figure}[t]
\centering
\includegraphics[width=1.0\linewidth]{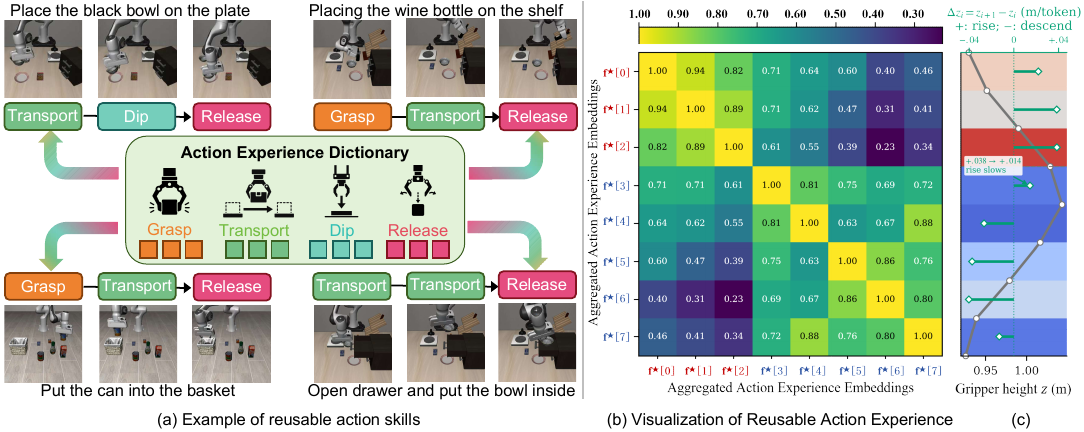}
\vspace{-7mm}
\caption{\textbf{(a)} Example of reusable action skills. 
\textbf{(b)} Visualization of reusable action experience in the action experience dictionary (AED). $\mathbf{f}^\star[0]\text{--}\mathbf{f}^\star[2]$ indicate the \texttt{transporting} action pattern, while $\mathbf{f}^\star[3]\text{--}\mathbf{f}^\star[7]$ encode the \texttt{dipping} action pattern in the AED. The visualization of the cosine similarities among action embeddings in the AED shows that learning to place the bowl on the plate involves substantial reuse of both the \texttt{transporting} and \texttt{dipping} action patterns. 
\textbf{(c)} \textcolor{gray}{Gray} and \textcolor{darkgreen}{green} lines show gripper height and its relative changes, respectively, during training to place the bowl on the plate, linking action patterns (e.g., \texttt{transporting} and \texttt{dipping}) to physical height. }
\vspace{-3mm}
\label{fig:motivation}
\end{figure}

To address the above challenges, we propose a novel learnable \underline{A}ction \underline{E}xperience \underline{D}ictionary (AED) that encodes historical manipulation trajectories as shared action embeddings for WAMs. \textbf{First}, we aggregate past physical actions over intervals aligned with visual observations and use a pretrained action tokenizer to retrieve task-relevant action embeddings from the AED. \textbf{Second}, the retrieved embeddings are pooled into compact temporal representations and conditioned on compressed historical visual features through cross-attention, enabling the resulting representations to capture both interaction context and action intent. Then, we prepend these visually conditioned action embeddings to the noisy action tokens, enabling the action expert to exploit inter-task relationships when predicting subsequent action chunks. \textbf{Third}, we introduce a motion-aware transition loss that supervises the prediction of visual feature changes over randomly sampled temporal intervals using the visual features at the start of each interval and the corresponding AED-conditioned hidden states. This loss encourages the retrieved action embeddings in AED to capture action-related motion while reducing reliance on irrelevant background. \textbf{Finally}, we evaluate the effectiveness of the proposed model by comparing it with baselines in simulation on LIBERO, RoboTwin, and LIBERO-Plus, as well as in real-world cross-embodiment experiments. The main contributions are listed below:
\begin{itemize}
\item We propose a novel learnable Action Experience Dictionary (AED) that encodes historical action experience into visually conditioned action embeddings, enabling our model to leverage underlying relationships among manipulation tasks to guide action prediction.

\item We incorporate action-relevant visual information into the action embeddings in the AED to obtain visually conditioned action embeddings, which combine task-related visual semantics with action experience to capture both task context and action intent.

\item We introduce a motion-aware transition loss that supervises visual feature change prediction over random temporal intervals, encouraging learned action embeddings in the AED to capture action-related motion and reducing reliance on irrelevant background cues.

\end{itemize}

\begin{figure}[t]
    \centering
    \includegraphics[width=1.0\linewidth]{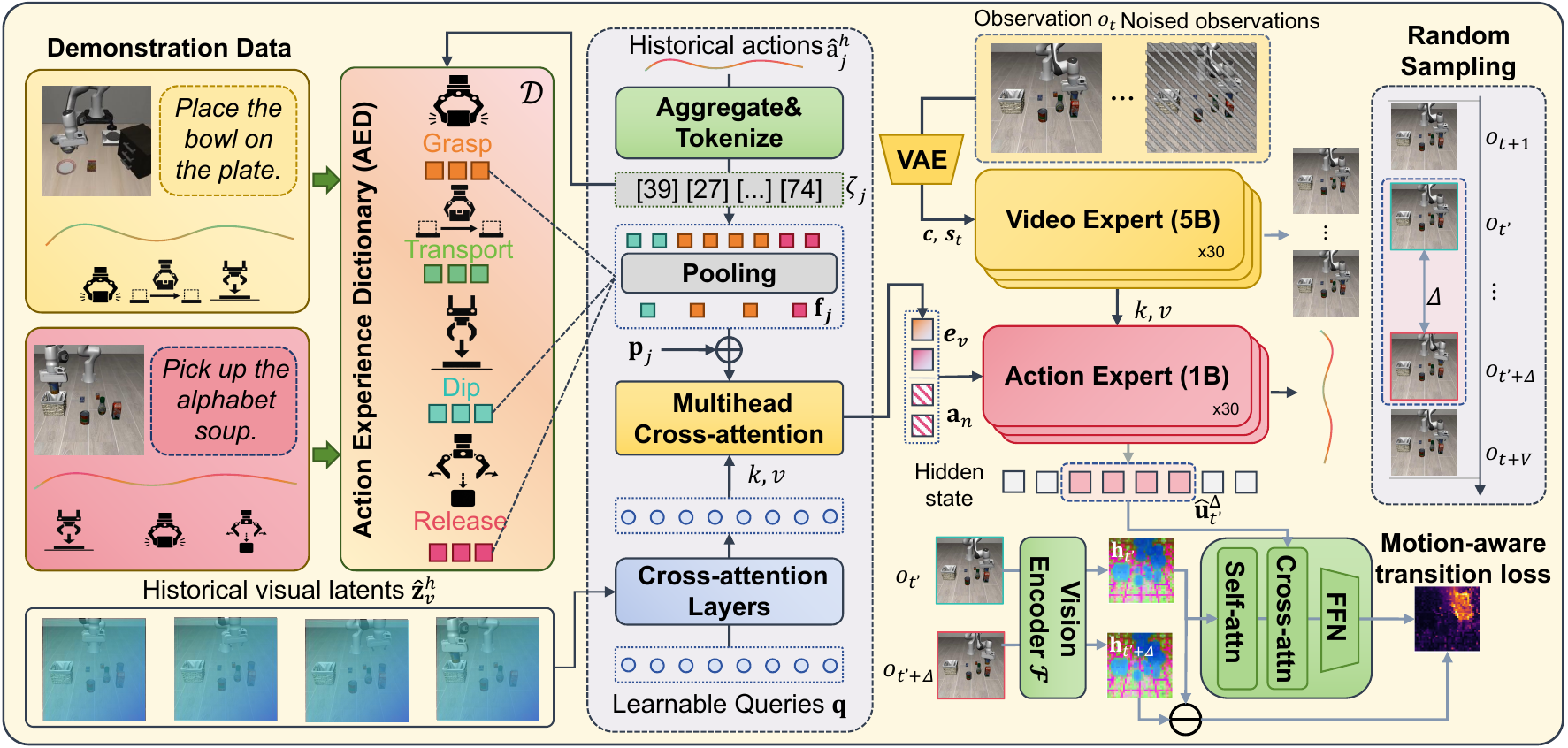}
    \vspace{-7mm}
    \caption{Overview of the proposed AED. After defining learnable action experience dictionary (AED) shared across tasks, we utilize the visually conditioned action embeddings to encode the task-relevant information and employ a motion-aware transition loss to encode action-relevant motion.
    }
    \vspace{-5mm}
    \label{fig:main_fig}
\end{figure}

\vspace{-2mm}
\section{Related work}
\vspace{-2mm}
\label{sec:related_work}

\textbf{Vision Language Action (VLA):}
Recent VLAs increasingly explore structured action representations and
visual dynamics for transferable robot control. FAST
\citep{pertsch2025fast} introduces frequency-space action tokenization for
efficient VLA training, while UniVLA \citep{bu2025univla} and ViPRA
\citep{routray2026vipra} learn task- or motion-centric latent actions from
heterogeneous videos to support transferable control. VLM2VLA
\citep{hancock2026vlm2vla} represents robot actions in a language-compatible
form to preserve pretrained VLM capabilities. DeFI \citep{zhang2026defi}
decouples forward visual dynamics and inverse action learning. However,
cross-task reuse of historical action experience remains underexplored.
Unlike UniVLA and ViPRA, which primarily learn transferable latent action spaces, our method explicitly visually conditions the retrieved action embeddings from the AED, and supervises visual feature change prediction over random temporal intervals to capture action-related motion and facilitate skill patterns reuse across tasks.

\textbf{World Action Models (WAMs):}
WAMs~\citep{WM_Survey} exploit visual prediction to capture physical and temporal structure for robot control. UniPi \citep{du2023unipi} performs planning through text-conditioned video generation and action extraction, while GR-2 \citep{cheang2024gr2} and DreamGen \citep{jang2025dreamgen} leverage video generative priors for manipulation. More recent WAMs couple visual dynamics and
action generation more directly: UWM \citep{zhu2025uwm} jointly models video and action diffusion, Motus \citep{bi2025motus} integrates understanding, video, and action experts, while LingBot-VA \citep{lingbot-va2026} and FastWAM \citep{yuan2026fastwam} shows that video
co-training can benefit control without test-time video generation.
Unlike existing WAMs that primarily learn task-specific behaviors through visual dynamics supervision, our method explicitly encodes historical action trajectories into shared action embeddings from AED to model underlying cross-task relationships, while using motion-aware transition supervision to emphasize action-related visual changes.

\section{Methodology}
\label{sec:method}

\subsection{Preliminaries}
\label{sec:method_setup}
\vspace{-2mm}
In embodied manipulation tasks~\citep{lingbot-va2026,palme}, robots aim to autonomously plan and make decisions based on task instructions $\mathbf{c}$ and visual observations $\mathbf{o}_t$ at time $t$. Following FastWAM~\citep{yuan2026fastwam}, we adopt flow matching to train control policies with both video and action experts. Let $\mathbf{x}_e$ ($e\in\{v,a\}$) denote the clean target, where $\mathbf{x}_v$ represents future video and $\mathbf{x}_a$ represents an action chunk. Given a Gaussian noise sample $\boldsymbol{\epsilon}\sim\mathcal{N}(\mathbf{0},\mathbf{I})$ and a flow time $\tau\in(0,1)$, we construct $\mathbf{x}_e^\tau=(1-\tau)\mathbf{x}_e+\tau\boldsymbol{\epsilon}$. Accordingly, the flow-matching objective $\mathcal{L}_{\mathrm{FM}}$ is defined as:

\begin{equation}
\mathcal L_{\mathrm{FM}} = \sum_{e\in\{v, a\}}\lambda_e\mathcal{L}_e;~~ 
    \mathcal{L}_e=\mathbb E_{\mathbf {x}_e,\boldsymbol\epsilon,\tau}\left[\left\|\pi_\theta(\mathbf{x}^{\tau}_e \mid \mathbf{o}_t, \mathbf{c}, \mathbf{s}_t)-(\boldsymbol\epsilon-\mathbf{x}_e)\right\|_2^2\right],
    \label{eq:flow_matching_loss}
\end{equation}
where $\pi_\theta(\cdot)$ denotes the policy parameterized by $\theta$, $\mathbf{s}_t$ is the robot's proprioceptive state at time $t$, and $\lambda_e=1.0$ is the balancing factor.
For joint training, we use $\pi_\theta(\cdot)$ to predict future video latents $\mathbf{z}_{t+1:t+V}$ corresponding to $V$ frames and an action chunk $\mathbf{a}_{t:t+H-1}$ of horizon $H$ to be executed.

\subsection{Action Experience Dictionary (AED)}
\label{sec:method_memory}
Generally, underlying relationships among manipulation tasks can provide valuable guidance for learning related skills. For example, when learning to place a cup on a shelf, a robot can draw on experience acquired from tabletop pick-and-place, such as grasping, transporting, and releasing objects, while adapting these behaviors to the target task. Conversely, experience gained from shelf placement can also benefit other tasks that involve these shared skills, enabling mutual knowledge transfer across related tasks. However, even during joint training, existing WAMs~\citep{yuan2026fastwam, lingbot-va2026, kim2026cosmospolicy} often learn task-specific behaviors without explicitly leveraging inter-task relationships, overlooking the potential of cross-task semantic connections to facilitate the learning of manipulation skills. This limitation motivates us to investigate \textbf{\textit{how to leverage reusable action experience shared across related tasks to improve manipulation performance. }}

To address the above limitation, we develop a novel learnable action experience dictionary (AED) to learn manipulation skills from historical action trajectories, as depicted in Fig.~2. Specifically, the proposed AED encodes historical physical action trajectories as a sequence of action embeddings to capture latent relationships across skills. Subsequently, for each training batch, we retrieve action embeddings that are highly relevant to the target task from the AED. These action embeddings are aggregated and then conditioned on the corresponding temporal visual embeddings through transformer blocks with cross-attention. Finally, we prepend the resulting action embeddings to the noisy action tokens within the action expert for prediction. This prepending strategy enables the model to exploit inter-task relationships for subsequent skill learning. 
To further encode action-relevant visual information into AED, we propose a motion-aware transition (MT) loss predicting temporal visual feature changes from the sampled observation and AED-guidance action hidden states.

\noindent {\color{black} \large $\triangleright$} \textbf{Construction of Action Experience Dictionary:}
Let $\mathcal{D}\in\mathbb{R}^{N\times d_a}$ denote a learnable Action Experience Dictionary (AED) containing $N$ randomly initialized action embeddings, each of dimension $d_a$. Notably, an action skill, such as grasping an object placed on a tabletop or inside a container, may comprise multiple action patterns with distinct motion characteristics, including approach direction, gripper height, displacement, and speed. 
Accordingly, each action embedding in the AED represents a specific action pattern, whereas a set of related action embeddings jointly characterizes the action skill.
During training, the proposed AED is shared across all manipulation tasks to facilitate the reuse of action experience, thereby leveraging cross-task relationships to benefit the target task. 
To encode $\mathcal{D}$, we use a pretrained action tokenizer $\Phi$ to convert historical physical action trajectories $\mathbf{a}^h\in\mathbb{R}^{H\times d_f}$ into a sequence of indices for retrieving task-relevant action embeddings from $\mathcal{D}$, where $d_f$ denotes the number of degrees of freedom. 
However, directly tokenizing these trajectories, each consisting of $H$ actions, results in a temporal frequency mismatch with the $V$ visual observations and incurs additional computational costs. To tackle this issue, we aggregate historical trajectories to obtain $\widehat{\mathbf{a}}^h\in\mathbb{R}^{V\times d_f}$, and define the $j$-th aggregated action $\widehat{\mathbf{a}}^h[j] \in\mathbb{R}^{d_f}$ as: 
\begin{equation}
\widehat{\mathbf{a}}^{h}[j] =
    \sum_{r=1}^{k}\mathbf{a}^{h}[(j-1)k+r]\odot\mathbf{m}[(j-1)k+r]
    +
    \mathbf{a}^{h}[jk]\odot(\mathbf{1}-\mathbf{m}[jk]),~~\forall j=1, \cdots, V,
    \label{eq:history_action_grouping}
\end{equation}
where $k\!=\!\frac{H}{V}$ is the number of historical actions between two consecutive observations, and $\mathbf{m}\in\{0,1\}^{H\times d_f}$ denotes a binary mask whose entries are set to $1$ for the arm control dimensions and $0$ for the gripper dimensions. Here, $\mathbf{m}[(j{-}1)k{+}r]$ and $\mathbf{m}[jk]$ represent the $((j{-}1)k{+}r)$-th and $jk$-th rows of $\mathbf{m}$, respectively. The same indexing convention applies to $\mathbf{a}[(j{-}1)k{+}r]$ and $\mathbf{a}[jk]$. In Eq.~(\ref{eq:history_action_grouping}), we accumulate the arm commands to summarize the motion executed over the interval of observations while retaining the final gripper command to preserve the gripper's terminal state.

Subsequently, we adopt $\Phi$ to map the $j$-th aggregated action $\widehat{\mathbf{a}}^h[j]$ to a sequence of indices $\zeta_j = \Phi(\widehat{\mathbf{a}}^h[j]) \in\mathbb{R}^{M}$, where $M$ denotes the number of retrieved action embeddings. Each index in $\zeta_j$ is then used to retrieve the corresponding action embedding from $\mathcal{D}$. These embeddings are averaged to obtain a compact action representation $\mathbf{f}_j \in\mathbb{R}^{d}$ for the $j$-th ($j=1,\ldots, V$) aggregated action $\widehat{\mathbf{a}}_j^h$:
\begin{align}
\mathbf{f}_j =
\mathbf{p}_j + \frac{1}{M} \sum_{l=1}^{M} \mathcal{D}(\zeta_j[l]), 
\label{eq}
\end{align}
where $\zeta_j[l]\in\mathbb{R}$ denotes the $l$-th index of $\zeta_j$, and $\mathcal{D}(\zeta_j[l])\in\mathbb{R}^{d}$ represents the $\zeta_j[l]$-th row of $\mathcal{D}$. $\mathbf{p}_j \in \mathbb{R}^{d_a}$ denotes the temporal positional encoding for the $j$-th action trajectory. Afterwards, we stack $\{\mathbf{f}_j\}_{j=1}^{V}$ in temporal order to obtain aggregated action embeddings $\mathbf{f}^\star \in \mathbb{R}^{V\times d_a}$:
\begin{equation}
\mathbf{f}^\star = [\mathbf{f}_1, \mathbf{f}_2, \ldots, \mathbf{f}_V].
\label{eq:selected_action_embeddings}
\end{equation}

\noindent {\color{black} \large $\triangleright$} \textbf{Visually Conditioned Action Embeddings:}
The action embeddings $\mathbf{f}^\star$ obtained in Eq.~(\ref{eq:selected_action_embeddings}) encode only the action patterns themselves (e.g., grasping and moving patterns), without capturing semantic information about task-relevant objects or their relationships with the target manipulation task. 
This may result in an incomplete understanding of the task context and action intent. To this end, we condition the action embeddings $\mathbf{f}^\star$ on visual information. Specifically, the historical visual latent embeddings comprise $R$ spatiotemporal patch tokens $\mathbf{z}_v^h \in \mathbb{R}^{R \times d_v}$, where $d_v$ denotes the dimension of visual embeddings. Since many of these tokens correspond to background content or content irrelevant to the interaction, directly feeding them into the action expert introduces redundancy and causes the input sequence length to grow with the observation horizon. Therefore, we compress $\mathbf{z}_v^h$ into a fixed number of latent visual tokens $\widehat{\mathbf{z}}_v^h \in \mathbb{R}^{L \times d_a}$ using $L$ learnable queries $\mathbf{q}\in\mathbb{R}^{L\times d_a}$:
\begin{equation}
\widehat{\mathbf{z}}_v^h=
\left[
\mathcal{A}_1
\oplus
\mathcal{A}_2
\oplus
\cdots
\oplus
\mathcal{A}_\psi
\right]
\mathbf{w}_o,
\;
\mathcal{A}_i
=
\sigma \!
(
\frac{
\mathbf{q}\mathbf{w}_q
\left(\mathbf{z}_v^h\mathbf{w}_k\right)^\top
}{
\sqrt{d_a/\psi}
})
(\mathbf{z}_v^h\mathbf{w}_v),
\; \forall i=1, \ldots, \psi,
\end{equation}
where $\oplus$ denotes concatenation along the feature dimension, $\psi$ is the number of cross-attention heads, and $\mathcal{A}_i \in\mathbb{R}^{L\times (d_a/\psi)}$ represents the $i$-th ($i=1, \ldots, \psi$) attention head. $\sigma(\cdot)$ is the softmax function. For each attention head, $\mathbf{w}_q \in \mathbb{R}^{d_a\times ({d_a}/{\psi})}$, $\mathbf{w}_k \in \mathbb{R}^{d_v\times ({d_a}/{\psi})}$, and $\mathbf{w}_v \in \mathbb{R}^{d_v\times ({d_a}/{\psi})}$ denote the linear projection matrices for the query, key, and value. Furthermore, $\mathbf{w}_o \in \mathbb{R}^{d_a\times d_a}$ indicates the output projection matrix used to fuse the features from all attention heads.

To incorporate the task-relevant visual context encoded in $\widehat{\mathbf{z}}_v^h$ into $\mathbf{f}^\star$ while preserving their action semantics, we fuse them using a dictionary encoder $\mathcal{E}$, implemented as a one-layer Transformer encoder with cross-attention. Here, $\mathbf{f}^\star$ serves as the query, while $\widehat{\mathbf{z}}_v^h$ provides the keys and values.
Using cross-attention, $\mathcal{E}$ generates visually conditioned action embeddings $\mathbf{e}_v \in \mathbb{R}^{V\times d_a}$, which are then concatenated with the noisy action tokens $\mathbf{a}_n \in \mathbb{R}^{H\times d_a}$ to obtain $\mathbf{a}^\star \in \mathbb{R}^{(H+V)\times d_a}$:
\begin{equation}
\begin{aligned}
\mathbf{a}^\star = [\mathbf{e}_v; \mathbf{a}_n], \quad \mathbf{e}_v = \mathcal{E}(\mathbf{f}^\star, \widehat{\mathbf{z}}_v^h),
\end{aligned}
\label{eq:action_prefix}
\end{equation}
where $[\cdot; \cdot]$ denotes concatenation along the token dimension. After concatenation, we feed $\mathbf{a}^\star$ to the action expert for action chunk prediction. Using the formulation in Eq.~(\ref{eq:action_prefix}), we integrate the task-relevant action experience retrieved from $\mathcal{D}$ and the associated visual information into the historical context, thereby enriching the context available for subsequent action prediction.

\noindent {\color{black} \large $\triangleright$} \textbf{Motion-Aware Transition Loss:}
Although visual conditioning incorporates historical scene information into the retrieved action embeddings for action prediction via Eq.~(\ref{eq:action_prefix}), the historical observations contain action-relevant objects and background content, and the flow-matching objective does not explicitly distinguish visual cues associated with action-dependent motion from incidental scene appearance. Thus, the resulting action embeddings in the AED $\mathcal{D}$ may retain background correlations that are less useful when reusing action experience across tasks. To encourage using action-relevant visual information, as illustrated in Fig.~\ref{fig:main_fig}, we develop a motion-aware transition (MT) loss that predicts temporal visual changes from the starting state and AED-conditioned action hidden states. Relatively stable background components can partially cancel in the feature difference, providing a supervision signal highlighting observable changes. Conditioning this prediction on action hidden states encourages the representations to capture object motion associated with the actions, helping reduce reliance on irrelevant background cues during action chunk prediction.

During training at time $t$, we randomly sample a future time step $t'\in\{t+1,\ldots,t+V-1\}$ and use a visual encoder $\mathcal{F}$ (e.g., LingBot-Vision~\citep{lingbot-vision2026} or DINOv3~\citep{simeoni2025dinov3}) to extract latent features $\mathbf{h}_{t'}=\mathcal{F}(\mathbf{o}_{t'})\in\mathbb{R}^{B\times d_z}$ from the future observation $\mathbf{o}_{t'}$. Here $B$ is the number of latent features and $d_z$ is the dimensionality of the latent features. We then sample a time window $\Delta \sim \operatorname{Unif}\{1,2,\ldots,t+V-t'\}$, where $\operatorname{Unif}$ is the discrete uniform distribution. 
After extracting the final-layer hidden states $\mathbf{u}_{t'}^\Delta \in \mathbb{R}^{\Delta\times d_a}$ from the action expert over the interval from $t'$ to $t'+ \Delta$, we propose the MT loss $\mathcal{L}_{\mathrm{MT}}$ that emphasizes visual objects whose motion is associated with the actions, thereby reducing the influence of irrelevant background cues on action chunk prediction:
\begin{equation}
\mathcal{L}_{\mathrm{MT}}
=
\frac{1}{Bd_z}
\left\|
\Delta\widehat{\mathbf{h}}_{t'}
-
\Delta\mathbf{h}_{t'}
\right\|_F^2, 
~~\Delta\widehat{\mathbf{h}}_{t'} =
\mathcal{G}(\mathbf{h}_{t'}, \mathbf{u}_{t'}^{\Delta}),
~~\Delta \mathbf{h}_{t'} =
\mathbf{h}_{t'+\Delta} - \mathbf{h}_{t'}, 
\label{eq:transition_loss}
\end{equation}
where $\Delta\widehat{\mathbf{h}}_{t'}\in\mathbb{R}^{B\times d_z}$ is the predicted visual feature change from time $t'$ to $t'+\Delta$. It is predicted using a three-layer predictor $\mathcal{G}$ with cross-attention between $\mathbf{u}^{\Delta}_{t'}$ and $\mathbf{h}_{t'}$. Here, cross-attention uses $\mathbf{h}_{t'}$ as the query and $\mathbf{u}^{\Delta}_{t'}$ as the keys and values. $\Delta\mathbf{h}_{t'}\in\mathbb{R}^{B\times d_z}$ is the ground-truth change in visual features from time $t'$ to $t'+\Delta$, and $\mathbf{h}_{t'+\Delta} = \mathcal{F}(\mathbf{o}_{t'+\Delta})$ is the latent representation of the future observation $\mathbf{o}_{t'+\Delta}$. Since the encoding of hidden state $\mathbf{u}_{t'}^\Delta$ incorporates the relevant action embeddings from $\mathcal{D}$, optimizing Eq.~(\ref{eq:transition_loss}) encourages $\mathcal{D}$ to capture action-related visual changes, focus on action-relevant objects, and reduce its reliance on irrelevant background during action prediction.

\begin{tcolorbox}[colback=green!3, colframe=black, boxrule=0.2 mm, boxsep=1pt, left=2pt, right=2pt, top=2pt, bottom=2pt, before skip=1pt, after skip=1pt]
\footnotesize
\setlength{\abovedisplayskip}{2pt}
\setlength{\belowdisplayskip}{2pt}
\setlength{\abovedisplayshortskip}{1pt}
\setlength{\belowdisplayshortskip}{1pt}
\begin{theorem}[Temporal Composition Error Bound]
\label{theorem1}
Let $V\geq3$ and assume consistent endpoint features from the same trajectory, with interval selection independent of the trajectory and auxiliary randomness. For any fixed parameters $\theta$ with finite transition risk, the sampled loss in Eq.~(\ref{eq:transition_loss}) is an unbiased estimator of the sampling-weighted risk $\mathcal R_{\mathrm{MT}}(\theta)$ defined below. The corresponding composition risk satisfies
\begin{equation}
    \mathcal C_{\mathrm{MT}}(\theta)
    \leq (V-1)(3V-4)\mathcal R_{\mathrm{MT}}(\theta).
    \label{eq:mt_random_composition_bound}
\end{equation}
Thus, vanishing expected MT loss implies vanishing mean-square composition error over the future intervals.
\end{theorem}
\end{tcolorbox}

\textbf{Proof.}
Let $X$ contain a trajectory and the auxiliary randomness defining its interval predictions, and let $\mathcal I=\{(a,b):t+1\leq a<b\leq t+V\}$. For $(a,b)\in\mathcal I$, define $\Delta\widehat{\mathbf h}_{a:b}=\mathcal G_\theta(\mathbf h_a,\mathbf u_a^{b-a})$, $\mathbf r_{a:b}=\Delta\widehat{\mathbf h}_{a:b}-(\mathbf h_b-\mathbf h_a)$, and $\ell_{a:b}(X;\theta)=\|\mathbf r_{a:b}\|_F^2/(Bd_z)$. The sampling distribution satisfies $p_{a:b}=1/[(V-1)(t+V-a)]\geq1/(V-1)^2$. Define $\bar\ell_{\mathrm{MT}}(X;\theta)=\sum_{(a,b)\in\mathcal I}p_{a:b}\ell_{a:b}(X;\theta)$ and $\mathcal R_{\mathrm{MT}}(\theta)=\mathbb E_X[\bar\ell_{\mathrm{MT}}(X;\theta)]$. For the sampled interval $I=(t',t'+\Delta)$, $\mathbb E_I[\ell_I(X;\theta)|X]=\bar\ell_{\mathrm{MT}}(X;\theta)$, and taking expectation over $X$ proves unbiasedness. 
For $t+1\leq a<b<c\leq t+V$, define $\mathbf r_{a:b:c}=\Delta\widehat{\mathbf h}_{a:b}+\Delta\widehat{\mathbf h}_{b:c}-\Delta\widehat{\mathbf h}_{a:c}$. Since the true feature differences telescope, $\mathbf r_{a:b:c}=\mathbf r_{a:b}+\mathbf r_{b:c}-\mathbf r_{a:c}$. Let $\kappa_{a:b:c}=p_{a:b}^{-1}+p_{b:c}^{-1}+p_{a:c}^{-1}$. Weighted Cauchy--Schwarz gives
\begin{equation}
\frac{\|\mathbf r_{a:b:c}\|_F^2}{Bd_z}
\leq
\kappa_{a:b:c}\bar\ell_{\mathrm{MT}}(X;\theta)
\leq
(V-1)(3V-4)\bar\ell_{\mathrm{MT}}(X;\theta).
\label{eq:mt_composition_pointwise}
\end{equation}
Indeed, $\kappa_{a:b:c}=(V-1)[3V-2(a-t)-(b-t)]\leq(V-1)(3V-4)$. Finally, define
$\mathcal C_{\mathrm{MT}}(\theta)=\mathbb E_X\!\left[\max_{t+1\leq a<b<c\leq t+V}\|\mathbf r_{a:b:c}\|_F^2/(Bd_z)\right]$.
Taking the maximum in Eq.~(\ref{eq:mt_composition_pointwise}) and then expectation over $X$ proves Eq.~(\ref{eq:mt_random_composition_bound}).
Theorem~\ref{theorem1} shows that the MT loss with randomly sampled intervals controls the temporal composition error of visual feature change predictions, providing a theoretical basis for consistent supervision across temporal scales. Lower compounding errors enable the model to better ignore extraneous disturbances that accumulate over long temporal horizons, such as changes in task-irrelevant objects arising from viewpoint shifts, manipulator motion, and incidental scene dynamics during task execution. Since these predictions depend on AED-conditioned action hidden states, AED learns the action experience by focusing on action-relevant visual information across different interaction stages and temporal scales through backpropagation.

\subsection{Training and Inference}
\label{sec:objective_inference}
\textbf{Training:}
We jointly train the video and action branches with the flow-matching objective $\mathcal{L}_{\mathrm{FM}}$ and the motion-aware transition loss $\mathcal{L}_{\mathrm{MT}}$. Therefore, the overall optimization $\mathcal{L}$ is defined as follows:
\begin{equation}
    \mathcal{L} = \mathcal L_{\mathrm{FM}} + \lambda_{m}\mathcal{L}_{\mathrm{MT}},
\end{equation}
where $\lambda_m=0.01$ denotes the balancing weight. For $\mathcal{L}_{\mathrm{FM}}$, we set $\lambda_e=1$ ($e\in\{v, a\}$) in Eq.~(\ref{eq:flow_matching_loss}).

\textbf{Inference:}
Video and action latents are initialized with Gaussian noise $\boldsymbol{\epsilon}\sim\mathcal{N}(\mathbf{0},\mathbf{I})$ and denoised using flow velocities predicted from the observation $\mathbf{o}_t$, proprioceptive state $\mathbf{s}_t$, task instruction $\mathbf{c}$, and AED embeddings $\mathbf{f}^\star$, with $K$ Euler steps of the flow-matching ordinary differential equation (ODE) yielding an action chunk of length $H$.
Deployment requires neither decoding the video latents into pixel-space frames nor evaluating the motion-aware transition predictor.

\vspace{-2mm}
\section{Experiments}
\label{sec:experimental_setup}

\begin{figure}[t]
    \centering
    \includegraphics[width=0.98\linewidth]{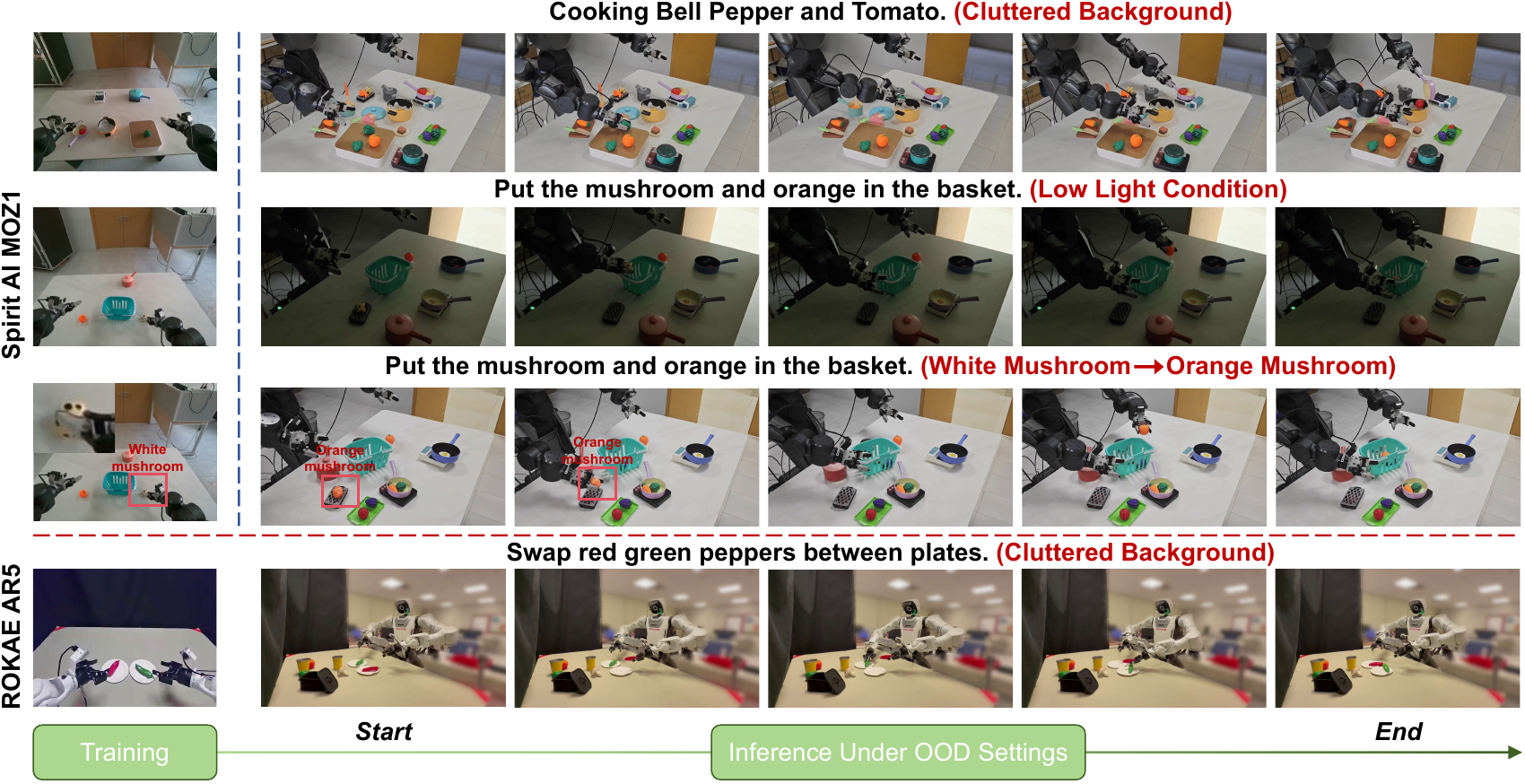}
    \vspace{-3mm}
    \caption{Visualization of manipulation tasks performed by our model in OOD settings. }
    \vspace{-6mm}
    \label{fig:ood_exp}
\end{figure}

\begin{figure}
    \centering
    \begin{subfigure}[t]{1.0\linewidth}
        \centering
        \includegraphics[width=\linewidth]{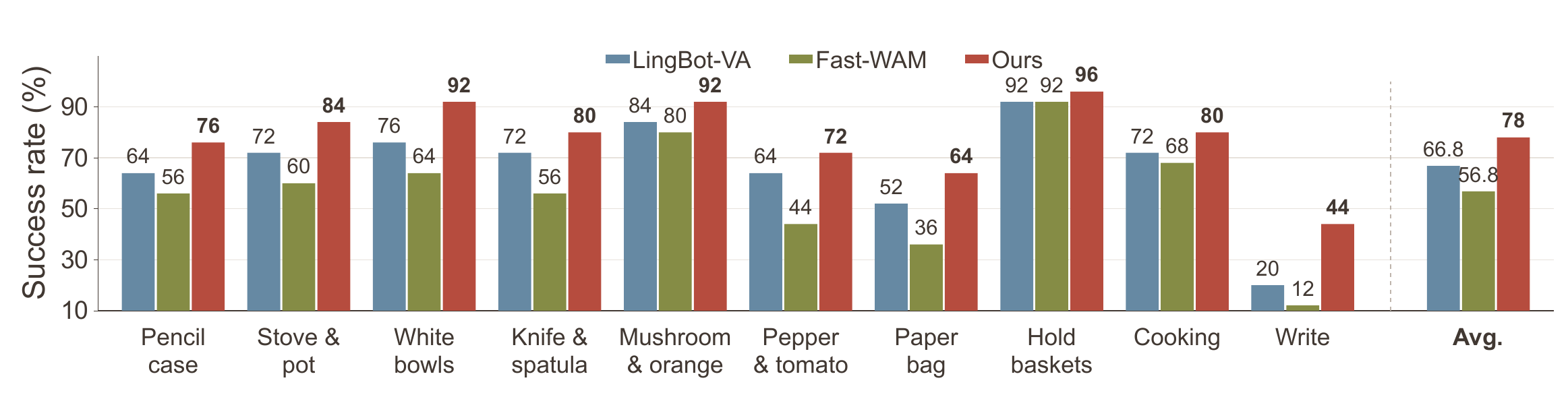}
        \vspace{-7mm}
        \caption{Comparison on the Spirit AI MOZ1 platform.}
        \vspace{-3mm}
        \label{fig:moz1}
    \end{subfigure}
    \hfill
    \begin{subfigure}[t]{0.58\linewidth}
        \centering
        \includegraphics[width=\linewidth]{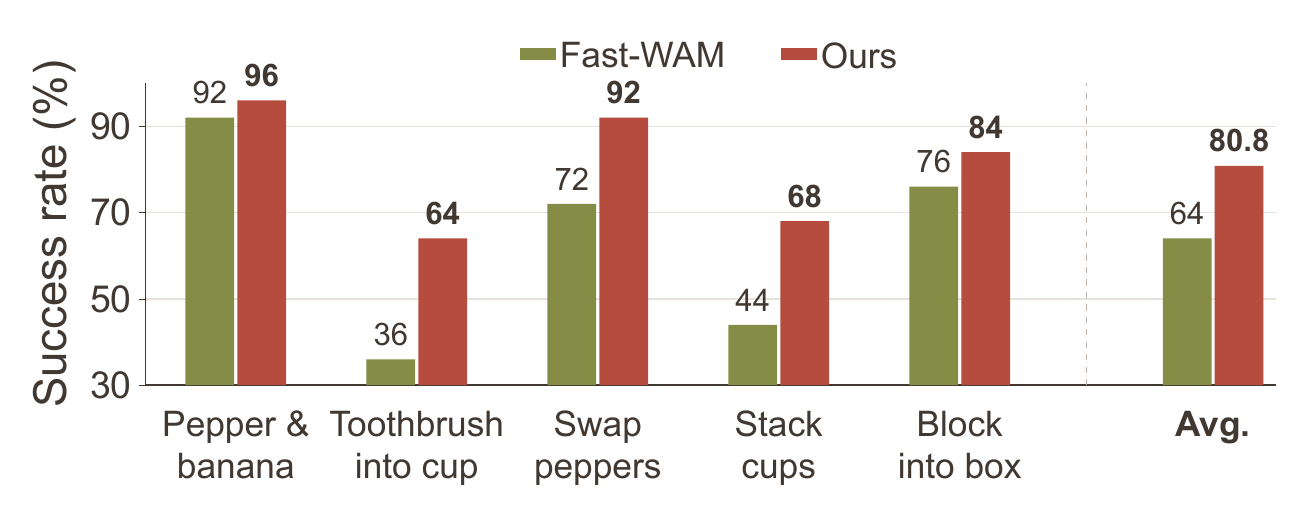}
        \vspace{-7mm}
        \caption{Comparison on ROKAE AR5 platform. }
        \label{fig:cross-embodiment}
    \end{subfigure}
    \hfill
    \begin{subfigure}[t]{0.4\linewidth}
        \centering
        \includegraphics[width=\linewidth]{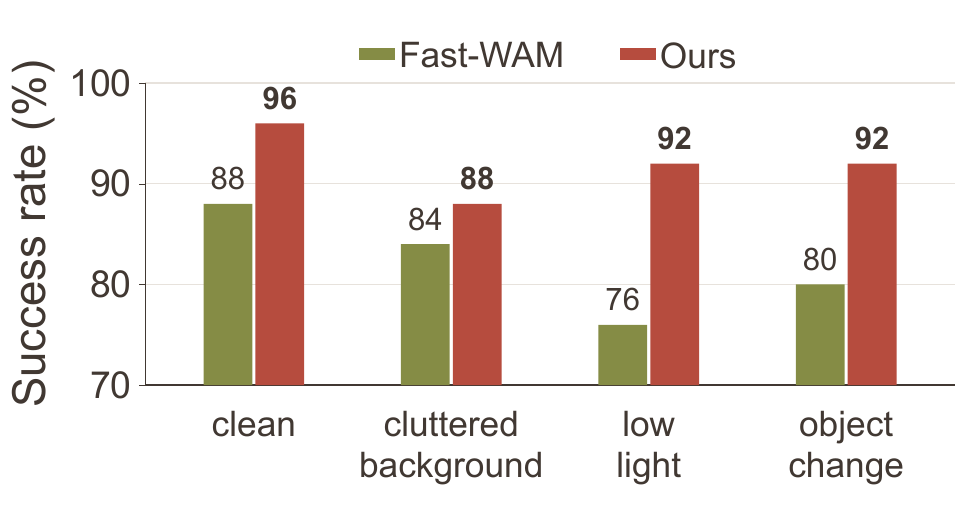}
        \vspace{-7mm}
        \caption{Comparison under different settings.}
        \label{fig:perturbations}
    \end{subfigure}
    \vspace{-3mm}
    \caption{Results on real-world manipulation tasks across robotic embodiments under OOD settings. }
    \vspace{-3mm}
    \label{fig:real_world}
    \vspace{-3mm}
\end{figure}

\subsection{Implementation details}
Following Fast-WAM~\citep{yuan2026fastwam}, the video expert (5B) is initialized from Wan2.2~\citep{wan2025}, retaining its video DiT, text encoder, and video VAE. The action expert (1B) adopts the same architectural design as the video branch, with its hidden dimension reduced to \(d_a=1024\). The action tokenizer follows the design of~\citet{pertsch2025fast}. All trainable parameters are optimized using AdamW for 10 epochs on LIBERO with 8 NVIDIA A100 GPUs and for 5 epochs on RoboTwin 2.0 with 32 NVIDIA H100 GPUs. We report success rates on various benchmarks, including LIBERO~\citep{liu2023libero}, RoboTwin 2.0~\citep{chen2026robotwin}, and LIBERO Plus~\citep{fei25libero-plus}. Physical experiments are conducted on two robotic platforms, Spirit AI MOZ1 and ROKAE AR5. Additional implementation details and evaluation settings are provided in the appendix.

\subsection{Main Comparison Results}

\textbf{Out-of-Distribution (OOD) Performance:}  
Since real-world environments involve various types of perturbations, we evaluate our model under OOD settings using the Spirit AI MOZ1 and ROKAE AR5 platforms. As shown in Fig.~\ref{fig:ood_exp}, we consider three types of OOD conditions: cluttered backgrounds, low light conditions, and unseen objects during training. The visualization of manipulation tasks performed by our model demonstrates its robustness under these OOD conditions. 
Additionally, as shown in Fig.~\ref{fig:real_world}(a)(b), we compare the success rates of our model with those of state-of-the-art baselines (e.g., LingBot-VA and Fast-WAM) across different robotic embodiments under OOD settings. As shown in Fig.~\ref{fig:real_world}(c), we further report the success rates on a representative real-world manipulation task, i.e., ``Put the mushroom and orange in the basket'', under different OOD conditions.
Our model consistently achieves higher success rates than the existing methods across different embodiments and OOD conditions, demonstrating the effectiveness of the proposed AED.

\begin{table*}[t]
  \centering
  \caption{Success rate (\%) on LIBERO and RoboTwin 2.0. LIBERO averages 50 rollouts per task over ten tasks per suite, and RoboTwin 2.0 averages 100 trials per task over 50 tasks in the `Clean'' and `Random'' environments. ``PT'' indicates whether robotic policy pretraining is used.
 }
  \vspace{-3mm}
  \label{tab:main_results}

  \small
  \setlength{\tabcolsep}{3.5pt}

  \begin{NiceTabular*}{\textwidth}{@{\extracolsep{\fill}}l|c|cccccccc@{}}
    \toprule
    \multirow{2}{*}{Method}
      & \multirow{2}{*}{PT}
      & \multicolumn{3}{c}{RoboTwin 2.0}
      & \multicolumn{5}{c}{LIBERO} \\
    \cmidrule(lr){3-5}
    \cmidrule(lr){6-10}
      & 
      & Clean & Random & Avg.
      & Spatial & Object & Goal & Long & Avg. \\
    \midrule
    OpenVLA~\citep{kim2024openvla}
      & \cmark
      & -- & -- & --
      & 84.7 & 88.4 & 79.2 & 53.7 & 76.5 \\

    $\pi_0$~\citep{black2024pi0}
      & \cmark
      & 65.9 & 58.4 & 62.2
      & 96.8 & \second{98.8} & 95.8 & 85.2 & 94.1 \\

    $\pi_{0.5}$~\citep{physicalintelligence2025pi05}
      & \cmark
      & 82.7 & 76.8 & 79.8
      & 98.8 & 98.2 & \second{98.0} & 92.4 & 96.9 \\

    Motus~\citep{bi2025motus}
      & \cmark
      & 88.7 & 87.0 & 87.8
      & 96.8 & \second{99.8} & 96.6 & 97.6 & 97.7 \\

    Fast-WAM~\citep{yuan2026fastwam}
      & \xmark
      & 91.9 & \second{91.8} & 91.8
      & 98.2 & \best{100.0} & 97.0 & 95.2 & 97.6 \\

    LingBot-VA~\citep{lingbot-va2026}
      & \cmark
      & \second{92.9} & 91.6 & \second{92.2}
      & 98.5 & 99.6 & 97.2 & \best{98.5} & \second{98.5} \\
    \midrule
    \rowcolor{gray!15}
    \textbf{AED} (\textbf{Ours})
      & \xmark
      & \best{93.2} & \best{92.3} & \best{92.8}
      & \best{99.0} & \best{100.0} & \best{98.4} & \second{97.8} & \best{98.8} \\
    \bottomrule
  \end{NiceTabular*}
\vspace{-4mm}
\end{table*}

\begin{table}[t]
  \centering
  \caption{Success rates (\%) under different perturbation types on LIBERO-Plus. ``PT'' indicates whether robotic policy pretraining is used. All success rates are computed over 10,030 trials.
  }
  \vspace{-3mm}
  \label{tab:libero_plus}

  \small
  \setlength{\tabcolsep}{3.5pt}

  \begin{NiceTabular*}{\linewidth}{
    @{\extracolsep{\fill}}
    l|c|*{8}{c}@{}
  }
    \toprule
    Method
      & PT
      & Cam.
      & Robot.
      & Lang.
      & Light.
      & Back.
      & Noise.
      & Layout.
      & Avg. \\
    \midrule

    OpenVLA~\citep{kim2024openvla}
      & \cmark
      & 0.8
      & 3.5
      & 23.0
      & 8.1
      & 34.8
      & 15.2
      & 28.5
      & 15.6 \\

    $\pi_0$~\citep{black2024pi0}
      & \cmark
      & 13.8
      & 6.0
      & 58.8
      & 85.0
      & 81.4
      & 79.0
      & 68.9
      & 53.6 \\

    $\pi_{0.5}$~\citep{physicalintelligence2025pi05}
      & \cmark
      & \second{75.4}
      & \best{77.5}
      & \best{85.6}
      & \second{96.9}
      & \best{94.6}
      & \second{89.7}
      & \best{85.7}
      & \second{85.7} \\

    % Motus~\citep{bi2025motus}
    %   & \cmark
    %   & --
    %   & --
    %   & --
    %   & --
    %   & --
    %   & --
    %   & --
    %   & -- \\

    Fast-WAM~\citep{yuan2026fastwam}
      & \xmark
      & 48.4
      & \second{75.7}
      & 66.8
      & 96.6
      & 68.5
      & 70.1
      & 76.5
      & 70.8 \\

    % LingBot-VA~\citep{lingbot-va2026}
    %   & \cmark
    %   & --
    %   & --
    %   & --
    %   & --
    %   & --
    %   & --
    %   & --
    %   & -- \\

    \midrule
    \rowcolor{gray!15}
    \textbf{AED} (\textbf{Ours})
      & \xmark
      & \best{93.4}
      & 72.9
      & \second{72.3}
      & \best{99.6}
      & \second{83.1}
      & \best{98.4}
      & \second{85.3}
      & \best{86.2} \\

    \bottomrule
  \end{NiceTabular*}
\vspace{-4mm}
\end{table}

% \FloatBarrier

\begin{table}[!t]
\begingroup
\centering

% Caption style
\captionsetup{
    position=top,
    skip=0pt,
    font=small,
    justification=raggedright,
    singlelinecheck=false
}

\setlength{\tabcolsep}{1.5pt}
\renewcommand{\arraystretch}{1}

% ============================================================
% Table 1
% ============================================================
\begin{minipage}[t]{0.30\linewidth}
    \vspace{0pt}

    \begin{minipage}[t][12pt][t]{\linewidth}
        \caption{Results on LIBERO-10.}
        \label{tab:ablation_libero_goal}
    \end{minipage}
    \par\nointerlineskip

    \footnotesize
    \begin{tabular*}{\linewidth}{
        @{\extracolsep{\fill}}lcccc@{}
    }
        \toprule
        \tblhead
        {Variant}
        & {MT}
        & {VC}
        & {AED}
        & Avg. \\
        \midrule

        \tblfour
        \textbf{Ours}
        & $\checkmark$
        & $\checkmark$
        & $\checkmark$
        & \best{97.4} \\

        \tblfour
        -- MT
        &
        & $\checkmark$
        & $\checkmark$
        & 96.8 \\

        \tblfour
        -- VC
        &
        &
        & $\checkmark$
        & 95.6 \\

        \tblfour
        -- AED
        &
        &
        &
        & 94.8 \\

        \bottomrule
    \end{tabular*}
\end{minipage}
\hfill
%
% ============================================================
% Table 2
% ============================================================
\begin{minipage}[t]{0.335\linewidth}
    \vspace{0pt}

    \begin{minipage}[t][12pt][t]{\linewidth}
        \caption{Inference cost on LIBERO.}
        \label{tab:latency_memory}
    \end{minipage}
    \par\nointerlineskip

    \footnotesize
    \begin{tabular*}{\linewidth}{
        @{\extracolsep{\fill}}ccc@{}
    }
        \toprule
        \tblhead
        {Steps}
        & {Latency (ms)}
        & {Memory} (GB) \\
        \midrule

        \tblfour
        1
        & $167.3\pm1.6$
        & 24.1 \\

        \tblfour
        4
        & $284.6\pm1.6$
        & 24.1 \\

        \tblfour
        10
        & $521.2\pm3.0$
        & 24.1 \\

        \tblfour
        15
        & $708.5\pm8.0$
        & 24.1 \\

        % \tblfive
        % 20
        % & $895.0\pm1.3$
        % & 24.13 \\

        \bottomrule
    \end{tabular*}
\end{minipage}
\hfill
%
% ============================================================
% Table 3
% ============================================================
\begin{minipage}[t]{0.325\linewidth}
    \vspace{0pt}

    \begin{minipage}[t][12pt][t]{\linewidth}
        \caption{Latency (ms) comparisons.}
        \label{tab:latency_success_goal}
    \end{minipage}
    \par\nointerlineskip

    \footnotesize
    \begin{tabular*}{\linewidth}{
        @{\extracolsep{\fill}}lcc@{}
    }
        \toprule
        \tblhead
        Method
        & LIBERO Goal
        & MOZ1 \\
        \midrule

        \tblfour
        $\pi_{0.5}$
        & 220.0
        & 494.9 \\

        \tblfour
        Motus
        & 2131.6
        & 2163.7 \\

        \tblfour
        Fast-WAM
        & 497.1
        & 513.8 \\

        \tblfour
        \textbf{Ours}
        & 518.7
        & 520.3 \\

        \bottomrule
    \end{tabular*}
\end{minipage}

\par\vspace{3pt}
\endgroup
\vspace{-6mm}
\end{table}

\textbf{Benchmark Comparison:}
Tab.~\ref{tab:main_results} shows that our method achieves the highest reported average success rates on both RoboTwin 2.0 (92.8\%) and LIBERO (98.8\%). On \textbf{RoboTwin 2.0}, our method achieves success rates of 93.2\% under the Clean setting and 92.3\% under the Random setting, yielding an average of 92.8\%. These results outperform the strongest reported baseline for each metric by 0.3, 0.5, and 0.6 percentage points, respectively. On \textbf{LIBERO}, our method outperforms the strongest baseline, Fast-WAM~\citep{yuan2026fastwam}, by 1.2 percentage points on average, with improvements of 0.8\%, 1.4\%, and 2.6\% on Spatial, Goal, and Long, respectively, while matching its 100.0\% success rate on Object. Our method ranks first on Spatial and Goal and ties for first on Object.
As shown in Tab.~\ref{tab:libero_plus}, our method achieves an overall success rate of 86.2\% on \textbf{LIBERO-Plus}, outperforming both Fast-WAM and $\pi_{0.5}$ without using embodied policy pretraining. It ranks first under the {Camera}, {Light}, and {Noise} perturbations and outperforms Fast-WAM in six of the seven categories, with {Robot} being the only exception. These results indicate that our robustness gains are concentrated in variations involving viewpoint, lighting, and sensor noise. 
\begin{figure}[t]
\centering
\includegraphics[width=0.98\linewidth]{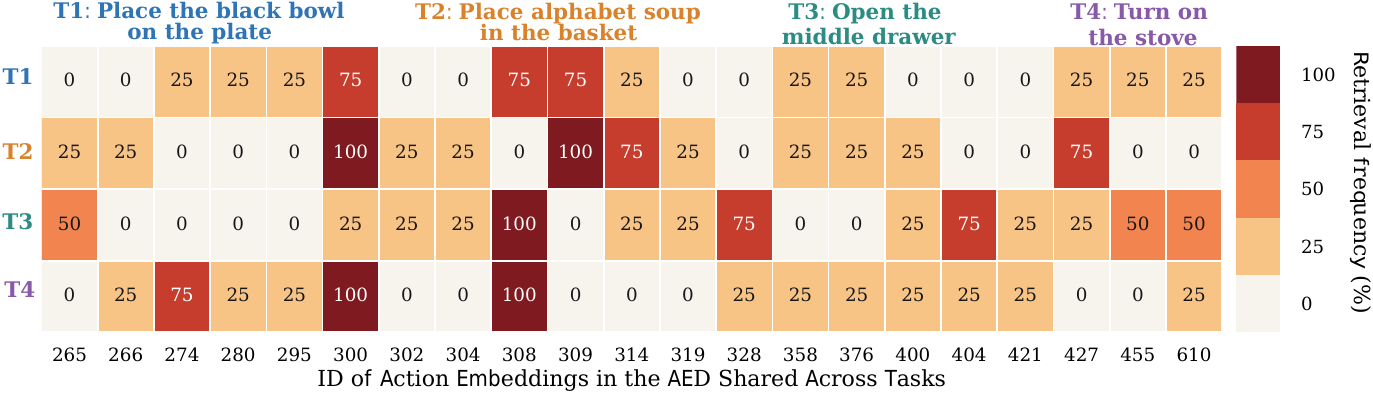}
\vspace{-3.5mm}
\caption{Analysis of reusing action experience embeddings across different manipulation tasks. }
\vspace{-5mm}
\label{fig:shared_tokens}
\end{figure}

% \vspace{-2mm}
\subsection{Ablation Study and Efficiency Analysis}
% \vspace{-2mm}
\label{sec:ablation}

\begin{figure}[t]
    \centering

    \begin{subfigure}[t]{0.60\linewidth}
        \centering
        \includegraphics[width=\linewidth]{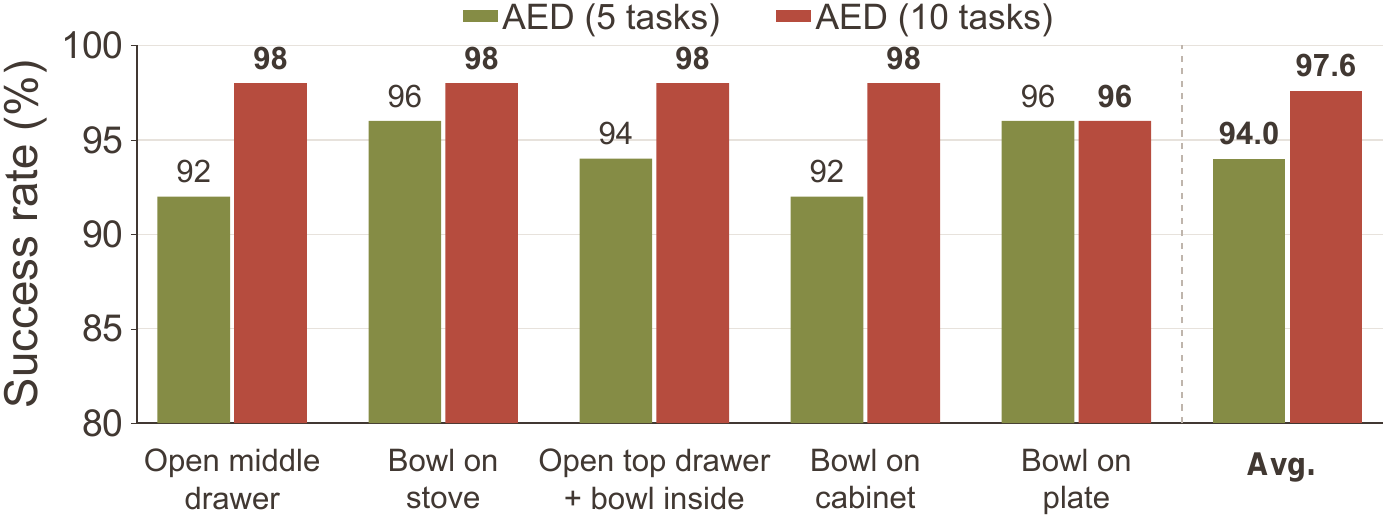}
        \vspace{-5mm}
        \caption{Quantitative Analysis of AED. }
        \label{fig:realworld-perturbations}
    \end{subfigure}
    % \hfill
    \hspace{0.01\linewidth}
    \begin{subfigure}[t]{0.36\linewidth}
        \centering
        \includegraphics[width=\linewidth]{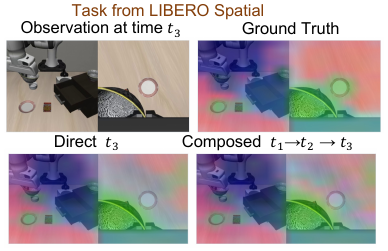}
        \vspace{-5mm}
        \caption{Direct versus composed transitions. }
        \label{fig:composed-visualization}
    \end{subfigure}
    \vspace{-3mm}
    \caption{\textbf{(a)} Quantitative performance evaluation of the proposed AED on LIBERO Goal. \textbf{(b)} Visualization of direct and composed transition predictions under supervision from the MT loss.}
    \vspace{-5mm}
    \label{fig:transfer}
\end{figure}

\textbf{Ablation Study:}
To evaluate the effectiveness of each component, we conduct ablation studies on the proposed motion-aware transition (MT) loss, visually conditioned action embeddings (VC), and learnable action experience dictionary (AED). As presented in Tab.~\ref{tab:ablation_libero_goal}, our full model achieves the highest success rate of \textbf{97.4\%} on LIBERO-10. Removing the MT loss reduces the success rate to 96.8\%, showing that MT loss supervision improves action prediction. Removing VC further degrades the performance, while further removing the AED results in a 2.6\% lower success rate than that of the full model. These ablation studies demonstrate the effectiveness of our model in reusing action experience across tasks to facilitate the learning of target manipulation tasks.

\textbf{Efficiency Analysis:} 
As shown in Tab.~\ref{tab:latency_memory}, we measure inference latency and peak GPU memory usage on LIBERO-10, with future frame decoding disabled and the transition predictor removed during inference. Increasing the number of denoising steps from 1 to 15 increases the latency from 167.3 to 708.5 ms, while peak memory usage remains constant at 24.1 GB. Tab.~\ref{tab:latency_success_goal} further compares the inference latency of our model on LIBERO Goal and a real-world robotic platform (e.g., Spirit AI MOZ1) against SOTA WAMs~\citep{yuan2026fastwam} and VLA models~\cite{physicalintelligence2025pi05} using a single NVIDIA A100 GPU, demonstrating comparable efficiency and strong potential for both simulated and real-world deployment. 
While maintaining efficiency comparable to that of the baselines, our model achieves significant performance improvements (see Tabs.~\ref{tab:main_results}--\ref{tab:libero_plus} and Fig.~\ref{fig:real_world}).

\subsection{Analysis of Action Experience Dictionary (AED)}
\vspace{-2mm}
To analyze how different manipulation tasks reuse action embeddings (i.e., skill patterns) shared across tasks,  Fig.~\ref{fig:shared_tokens} visualizes the reuse frequencies of shared embeddings across four randomly selected LIBERO tasks. We observe that many action embeddings are frequently reused across these tasks, indicating that the proposed AED can encode action-relevant skill patterns shared across tasks and leverage them to improve the performance of target manipulation tasks during training. A high frequency of reusing the same action embeddings across different tasks suggests stronger cross-task relationships. The proposed model captures such inter-task relationships through shared action embeddings in the AED and leverages them to facilitate future action chunk prediction.

To quantitatively evaluate the efficacy of reusable action experience across tasks, as shown in Fig.~\ref{fig:transfer}(a), we randomly select five LIBERO Goal tasks as target tasks and treat the remaining five as auxiliary source tasks. Despite having different goals, these tasks share reusable skill patterns, such as grasping, transporting, and placing objects. We train one model on five target tasks and another on the same target tasks plus five auxiliary source tasks, evaluating both on the same target tasks. In Fig.~\ref{fig:transfer}(a), training the proposed AED on all ten tasks consistently yields higher success rates. Such improvement suggests positive transfer from the additional related tasks, i.e., shared skill patterns learned from the five auxiliary source tasks benefit the learning of the five target tasks.

% \vspace{-2mm}
\subsection{Analysis of Motion-Aware Transition (MT) Loss}
% \vspace{-2mm}
To qualitatively evaluate the MT loss, we visualize its compositional generalization ability in Fig.~\ref{fig:transfer}(b). ``Direct $t_3$'' predicts the feature change from $t_1$ to $t_3$ directly, while ``Composed $t_1 \rightarrow t_2 \rightarrow t_3$'' predicts it through two consecutive transitions. The two methods produce similar visualization results and identify action-relevant objects (e.g., grippers and bowls), indicating effective transition composition. With the guidance of MT loss, the policy learns to suppress background interference.
\textcolor{blue}{Additional analyses of sampling strategies, action aggregation, visual encoders, prefix designs, action experience embeddings, and learnable visual queries are provided in the appendix. }

% \vspace{-3mm}
\section{Conclusion}
% \vspace{-3mm}
In this paper, we propose a novel Action Experience Dictionary (AED) for learning skills from historical trajectories. AED encodes trajectories into shared action embeddings and conditions them on visual context to capture task-relevant interactions. We further propose a motion-aware transition loss that predicts visual feature changes over random temporal intervals, encouraging action-centric representations. Experiments across simulation benchmarks and real-world cross-embodiment evaluations demonstrate improved manipulation performance over baseline approaches.

\bibliography{iclr2027_conference}
\bibliographystyle{iclr2027_conference}

\newpage
\appendix
\section{Reproducibility and Robotic Manipulation Demos}
The real-world and simulation demos and reproducible code are available on the anonymous project: \url{https://anonymous.4open.science/w/supplementary_materials-4586}.

\section{Limitation}
Our method has two main limitations. Firstly, a frozen visual model needs to be loaded, which may slow down the optimization speed during training, although this can be solved by caching visual features in advance. Secondly, the finite history window limits access to longer-term dependencies. We will address this issue in our future work.

\section{Hyperparameters and Implementation Details}
\label{app:implementation_details}
\subsection{Model and Representation Configuration}
\label{app:model_configuration}
\paragraph{Backbone and trainable modules.}
We initialize the video expert from Wan2.2-TI2V-5B from pretrained Wan2.2 weights, with the ActionDiT action expert initialized by linear interpolation of the parameters from the video expert. Both experts contain 30 layers, with hidden dimensions of 3,072 and 1,024, respectively. We train both experts, the Action Experience Dictionary (AED), the visual memory and action--visual fusion modules, and the transition predictor. The video VAE and the LingBot-vision encoder remain frozen. We employ delta pose to control the robot.

\paragraph{Action Experience Dictionary.}
The AED contains 2,048 valid entries of dimension 1,024, plus a separate padding (PAD) entry. FAST+ token indices retrieve dictionary embeddings, which are aggregated with temporal positional encoding as described in the main paper.

\paragraph{Visual conditioning.}
Historical visual features are extracted by the VAE and compressed into 128 embeddings using two memory-compression layers. A single action--visual fusion block combines the visual memory with the action history to produce eight historical prefix tokens for the action expert. The frozen LingBot encoder provides the visual targets for transition supervision.

\subsection{Optimization and Training Configuration}
\label{app:training_configuration}
We use AdamW with a learning rate of $10^{-4}$, $\beta_1=0.9$, $\beta_2=0.95$, weight decay of 0.01, and a gradient-clipping threshold of 1. The schedule includes 10\% warmup updates followed by cosine decay to $10^{-6}$ at the end of the planned run. Training uses 8 NVIDIA A100 GPUs and 32 NVIDIA H100 GPUs with eight samples per GPU. We use BF16 precision and DeepSpeed ZeRO-1. The full training schedule comprises 10 epochs for LIBERO and 5 epochs for RoboTwin.
The video and action flow-matching losses each have a weight of 1. The MT loss weight is 0.01.

\subsection{Observation and Action Preprocessing}
\label{app:preprocessing_details}
For single-arm manipulation, the input includes a third-person view and a wrist-camera view. For dual-arm manipulation, we use three views: a head-mounted view and one wrist-camera view for each arm. Each image is resized to $224\times224$, and the views are concatenated horizontally.

Single-arm actions are represented by seven-dimensional vectors. Actions are normalized using min--max scaling. The history contains 32 action steps, which are aggregated into eight consecutive groups of four steps before FAST+ tokenization.

\subsection{Inference and Deployment}
\label{app:inference_deployment}
We use NVIDIA A100 GPUs for inference. The policy predicts a chunk of 32 future actions using 10 flow-matching ordinary differential equation (ODE) integration steps. We execute the first 10 predicted actions before replanning, at a control frequency of 30~Hz.

\begin{table*}[t]
    \centering
    \small
    \begin{minipage}{0.98\textwidth}
    \refstepcounter{algorithm}\label{alg:aed_wam}
    \hrule height 0.6pt
    \vspace{2pt}
    \noindent\textbf{Algorithm \thealgorithm\quad Pipeline of the Proposed AED-WAM}
    \vspace{2pt}\hrule
    \medskip
    \noindent\textbf{Input:} Training batch $\mathcal B$ with instruction $\mathbf c$, observations $\mathbf o$, proprioceptive states $\mathbf s$, historical actions $\mathbf a^h$, and flow-matching targets $(\mathbf x_v,\mathbf x_a)$. At inference, use $(\mathbf c,\mathbf o_t,\mathbf s_t,\mathbf a^h,\mathbf z_v^h)$, where $\mathbf z_v^h=\operatorname{VAE}(\mathbf o^h)$.\\
    \noindent\textbf{Output:} Action chunk $\widehat{\mathbf a}_{t:t+H-1}$.\par\smallskip
    \begin{tabular}{@{}r@{\quad}p{0.91\textwidth}@{}}
    \multicolumn{2}{@{}l}{\textbf{$\triangleright$ Training}} \\
    1: & Set $k=H/V$; compute $\widehat{\mathbf a}^{h}[j]$ by Eq.~(\ref{eq:history_action_grouping}), tokenize $\zeta_j=\Phi(\widehat{\mathbf a}^{h}[j])$, and retrieve the temporally ordered action embeddings $\mathbf f^\star=[\mathbf p_j+M^{-1}\sum_l\mathcal D(\zeta_j[l])]_{j=1}^{V}$. \\
    2: & Compress historical visual latents $\mathbf z_v^h$ to $\widehat{\mathbf z}^{h}_v$ and obtain visually conditioned action embeddings $\mathbf e_v=\mathcal E(\mathbf f^\star,\widehat{\mathbf z}^{h}_v)$; prepend them to noisy action tokens (Eq.~(\ref{eq:action_prefix})). \\
    3: & Sample $\tau$ and $\boldsymbol\epsilon_e$; set $\mathbf x_e^\tau=(1-\tau)\mathbf x_e+\tau\boldsymbol\epsilon_e$ for $e\in\{v,a\}$ and compute $\mathcal L_{\rm FM}$ (Eq.~(\ref{eq:flow_matching_loss})). \\
    4: & Sample $(t',\Delta)$; predict $\Delta\widehat{\mathbf h}_{t'}=\mathcal G(\mathbf h_{t'},\mathbf u_{t'}^\Delta)$ and compute $\mathcal L_{\rm MT}$ against $\Delta\mathbf h_{t'}=\mathbf h_{t'+\Delta}-\mathbf h_{t'}$ (Eq.~(\ref{eq:transition_loss})). \\
    5: & Update trainable parameters by $\mathcal L=\mathcal L_{\rm FM}+0.01\,\mathcal L_{\rm MT}$; keep the video VAE and target encoder $\mathcal F$ frozen. \\
    \multicolumn{2}{@{}l}{\textbf{$\triangleright$ Inference}} \\
    6: & Receive $(\mathbf c,\mathbf o_t,\mathbf s_t,\mathbf a^h,\mathbf z_v^h)$ and repeat Steps 1--2 to construct $\mathbf e_v$. \\
    7: & Initialize the video and action latents with Gaussian noise at $\tau_0=1$; integrate the joint flow-matching ODE backward to $\tau_K=0$ with $K$ Euler steps, conditioned on $(\mathbf o_t,\mathbf s_t,\mathbf c,\mathbf e_v)$. Only the resulting action chunk is executed. \\
    8: & Decode $\widehat{\mathbf a}_{t:t+H-1}$, execute the first 10 actions at 30~Hz, append the new history and latents, and repeat Step 6. \\
    \end{tabular}
    \vspace{2pt}\hrule height 0.6pt
    \end{minipage}
\end{table*}

\section{Per-Task Benchmark Results}
\label{app:per_task_results}

\subsection{LIBERO}
% BEGIN GENERATED LIBERO40 RESULTS
\label{app:libero40_results}
Table~\ref{tab:libero40_per_task} reports the 42K checkpoint's success rates on standard LIBERO, with 1,976 successes over 2,000 trials (98.80\%). Task IDs match the S0--S9, O0--O9, G0--G9, and L0--L9 definitions in the task protocol.

\begingroup
\setlength{\LTpre}{6pt}
\setlength{\LTpost}{6pt}
\setlength{\LTcapwidth}{\textwidth}

\footnotesize
\setlength{\tabcolsep}{4pt}
\renewcommand{\arraystretch}{1.02}

\begin{longtable}{@{}>{\raggedright\arraybackslash}p{0.045\textwidth}>{\raggedright\arraybackslash}p{0.70\textwidth}rr@{}}
\caption{Per-task success rates on all 40 standard LIBERO tasks, grouped by suite. The same 42,000-step checkpoint is evaluated with seed 3407, 50 trials per task, and replanning after ten executed actions. SR denotes success rate in percent. Each suite average covers 500 trials; the overall average covers 2,000 trials.}\label{tab:libero40_per_task}\\
\toprule
\textbf{ID} & \textbf{Task instruction} & \textbf{Successes} & \textbf{SR (\%)} \\
\midrule
\endfirsthead
\caption[]{LIBERO per-task results (continued).}\\
\toprule
\textbf{ID} & \textbf{Task instruction} & \textbf{Successes} & \textbf{SR (\%)} \\
\midrule
\endhead
\midrule
\multicolumn{4}{r}{Continued on next page}\\
\endfoot
\bottomrule
\endlastfoot
\multicolumn{4}{@{}l}{\textbf{LIBERO-Spatial}}\\
S0 & pick up the black bowl between the plate and the ramekin and place it on the plate & 50/50 & 100.00 \\
S1 & pick up the black bowl next to the ramekin and place it on the plate & 50/50 & 100.00 \\
S2 & pick up the black bowl from table center and place it on the plate & 50/50 & 100.00 \\
S3 & pick up the black bowl on the cookie box and place it on the plate & 48/50 & 96.00 \\
S4 & pick up the black bowl in the top drawer of the wooden cabinet and place it on the plate & 49/50 & 98.00 \\
S5 & pick up the black bowl on the ramekin and place it on the plate & 49/50 & 98.00 \\
S6 & pick up the black bowl next to the cookie box and place it on the plate & 50/50 & 100.00 \\
S7 & pick up the black bowl on the stove and place it on the plate & 49/50 & 98.00 \\
S8 & pick up the black bowl next to the plate and place it on the plate & 50/50 & 100.00 \\
S9 & pick up the black bowl on the wooden cabinet and place it on the plate & 50/50 & 100.00 \\
\multicolumn{2}{@{}l}{\textbf{Suite average}} & \textbf{495/500} & \textbf{99.00} \\
\midrule
\multicolumn{4}{@{}l}{\textbf{LIBERO-Object}}\\
O0 & pick up the alphabet soup and place it in the basket & 50/50 & 100.00 \\
O1 & pick up the cream cheese and place it in the basket & 50/50 & 100.00 \\
O2 & pick up the salad dressing and place it in the basket & 50/50 & 100.00 \\
O3 & pick up the bbq sauce and place it in the basket & 50/50 & 100.00 \\
O4 & pick up the ketchup and place it in the basket & 50/50 & 100.00 \\
O5 & pick up the tomato sauce and place it in the basket & 50/50 & 100.00 \\
O6 & pick up the butter and place it in the basket & 50/50 & 100.00 \\
O7 & pick up the milk and place it in the basket & 50/50 & 100.00 \\
O8 & pick up the chocolate pudding and place it in the basket & 50/50 & 100.00 \\
O9 & pick up the orange juice and place it in the basket & 50/50 & 100.00 \\
\multicolumn{2}{@{}l}{\textbf{Suite average}} & \textbf{500/500} & \textbf{100.00} \\
\midrule
\multicolumn{4}{@{}l}{\textbf{LIBERO-Goal}}\\
G0 & open the middle drawer of the cabinet & 49/50 & 98.00 \\
G1 & put the bowl on the stove & 48/50 & 96.00 \\
G2 & put the wine bottle on top of the cabinet & 48/50 & 96.00 \\
G3 & open the top drawer and put the bowl inside & 50/50 & 100.00 \\
G4 & put the bowl on top of the cabinet & 49/50 & 98.00 \\
G5 & push the plate to the front of the stove & 50/50 & 100.00 \\
G6 & put the cream cheese in the bowl & 49/50 & 98.00 \\
G7 & turn on the stove & 50/50 & 100.00 \\
G8 & put the bowl on the plate & 50/50 & 100.00 \\
G9 & put the wine bottle on the rack & 49/50 & 98.00 \\
\multicolumn{2}{@{}l}{\textbf{Suite average}} & \textbf{492/500} & \textbf{98.40} \\
\midrule
\multicolumn{4}{@{}l}{\textbf{LIBERO-Long (LIBERO-10)}}\\
L0 & put both the alphabet soup and the tomato sauce in the basket & 49/50 & 98.00 \\
L1 & put both the cream cheese box and the butter in the basket & 50/50 & 100.00 \\
L2 & turn on the stove and put the moka pot on it & 49/50 & 98.00 \\
L3 & put the black bowl in the bottom drawer of the cabinet and close it & 47/50 & 94.00 \\
L4 & put the white mug on the left plate and put the yellow and white mug on the right plate & 48/50 & 96.00 \\
L5 & pick up the book and place it in the back compartment of the caddy & 50/50 & 100.00 \\
L6 & put the white mug on the plate and put the chocolate pudding to the right of the plate & 48/50 & 96.00 \\
L7 & put both the alphabet soup and the cream cheese box in the basket & 50/50 & 100.00 \\
L8 & put both moka pots on the stove & 49/50 & 98.00 \\
L9 & put the yellow and white mug in the microwave and close it & 49/50 & 98.00 \\
\multicolumn{2}{@{}l}{\textbf{Suite average}} & \textbf{489/500} & \textbf{97.80} \\
\midrule
\multicolumn{2}{@{}l}{\textbf{Overall average}} & \textbf{1976/2000} & \textbf{98.80} \\

\end{longtable}
\endgroup

\subsection{RoboTwin 2.0}
\label{app:robotwin_per_task}

Table~\ref{tab:robotwin_per_task} compares our method with Fast-WAM~\citep{yuan2026fastwam} on all 50 RoboTwin 2.0 tasks. Our evaluation uses a single checkpoint and 100 trials per task in each of the Clean and Random settings, yielding 10,000 trials in total. Our method succeeds in 4,659/5,000 Clean trials and 4,617/5,000 Random trials, corresponding to 93.18\% and 92.34\%, respectively, and an overall success rate of 92.76\%.

\begingroup
\setlength{\LTpre}{6pt}
\setlength{\LTpost}{6pt}
\setlength{\LTcapwidth}{\textwidth}

  \footnotesize
  \setlength{\tabcolsep}{5pt}
  \renewcommand{\arraystretch}{1.02}
  \begin{longtable}{@{\extracolsep{\fill}}lrrrrrr@{}}
    \caption{Per-task success rate (\%) on RoboTwin 2.0. Fast-WAM entries are reported baseline results; ours use 100 trials per task and setting with unseen instructions. Avg. averages Clean and Random, and the final row averages all 50 tasks. Bold indicates the better result between the two methods for each metric, including ties.}\label{tab:robotwin_per_task}\\
\toprule
    & \multicolumn{3}{c}{Fast-WAM} & \multicolumn{3}{c}{Ours} \\
    \cmidrule(lr){2-4}\cmidrule(lr){5-7}
    Task & Clean & Random & Avg. & Clean & Random & Avg. \\
    \midrule
\endfirsthead
\caption[]{RoboTwin 2.0 per-task results (continued).}\\
\toprule
    & \multicolumn{3}{c}{Fast-WAM} & \multicolumn{3}{c}{Ours} \\
    \cmidrule(lr){2-4}\cmidrule(lr){5-7}
    Task & Clean & Random & Avg. & Clean & Random & Avg. \\
    \midrule
\endhead
\midrule
\multicolumn{7}{r}{Continued on next page}\\
\endfoot
\bottomrule
\endlastfoot
    \texttt{adjust\_bottle} & \textbf{100.00} & \textbf{100.00} & \textbf{100.00} & \textbf{100.00} & 99.00 & 99.50 \\
    \texttt{beat\_block\_hammer} & \textbf{99.00} & 97.00 & 98.00 & 98.00 & \textbf{99.00} & \textbf{98.50} \\
    \texttt{blocks\_ranking\_rgb} & \textbf{100.00} & \textbf{100.00} & \textbf{100.00} & \textbf{100.00} & 99.00 & 99.50 \\
    \texttt{blocks\_ranking\_size} & \textbf{94.00} & \textbf{98.00} & \textbf{96.00} & 92.00 & 96.00 & 94.00 \\
    \texttt{click\_alarmclock} & \textbf{100.00} & \textbf{100.00} & \textbf{100.00} & \textbf{100.00} & 99.00 & 99.50 \\
    \texttt{click\_bell} & \textbf{100.00} & \textbf{100.00} & \textbf{100.00} & \textbf{100.00} & \textbf{100.00} & \textbf{100.00} \\
    \texttt{dump\_bin\_bigbin} & 97.00 & 96.00 & 96.50 & \textbf{98.00} & \textbf{98.00} & \textbf{98.00} \\
    \texttt{grab\_roller} & \textbf{100.00} & \textbf{100.00} & \textbf{100.00} & \textbf{100.00} & \textbf{100.00} & \textbf{100.00} \\
    \texttt{handover\_block} & \textbf{95.00} & 81.00 & 88.00 & 94.00 & \textbf{84.00} & \textbf{89.00} \\
    \texttt{handover\_mic} & \textbf{99.00} & \textbf{100.00} & \textbf{99.50} & \textbf{99.00} & 99.00 & 99.00 \\
    \texttt{hanging\_mug} & \textbf{58.00} & \textbf{62.00} & \textbf{60.00} & 56.00 & 58.00 & 57.00 \\
    \texttt{lift\_pot} & \textbf{100.00} & \textbf{100.00} & \textbf{100.00} & \textbf{100.00} & \textbf{100.00} & \textbf{100.00} \\
    \texttt{move\_can\_pot} & 90.00 & 88.00 & 89.00 & \textbf{99.00} & \textbf{98.00} & \textbf{98.50} \\
    \texttt{move\_pillbottle\_pad} & \textbf{100.00} & 99.00 & 99.50 & \textbf{100.00} & \textbf{100.00} & \textbf{100.00} \\
    \texttt{move\_playingcard\_away} & \textbf{100.00} & \textbf{100.00} & \textbf{100.00} & 99.00 & \textbf{100.00} & 99.50 \\
    \texttt{move\_stapler\_pad} & 77.00 & 64.00 & 70.50 & \textbf{86.00} & \textbf{76.00} & \textbf{81.00} \\
    \texttt{open\_laptop} & \textbf{98.00} & \textbf{100.00} & \textbf{99.00} & 96.00 & \textbf{100.00} & 98.00 \\
    \texttt{open\_microwave} & 62.00 & 45.00 & 53.50 & \textbf{86.00} & \textbf{68.00} & \textbf{77.00} \\
    \texttt{pick\_diverse\_bottles} & 80.00 & 85.00 & 82.50 & \textbf{82.00} & \textbf{87.00} & \textbf{84.50} \\
    \texttt{pick\_dual\_bottles} & \textbf{100.00} & \textbf{96.00} & \textbf{98.00} & \textbf{100.00} & \textbf{96.00} & \textbf{98.00} \\
    \texttt{place\_a2b\_left} & 95.00 & 93.00 & 94.00 & \textbf{100.00} & \textbf{95.00} & \textbf{97.50} \\
    \texttt{place\_a2b\_right} & 93.00 & \textbf{99.00} & 96.00 & \textbf{96.00} & 97.00 & \textbf{96.50} \\
    \texttt{place\_bread\_basket} & 91.00 & 93.00 & 92.00 & \textbf{92.00} & \textbf{97.00} & \textbf{94.50} \\
    \texttt{place\_bread\_skillet} & 90.00 & \textbf{93.00} & 91.50 & \textbf{92.00} & \textbf{93.00} & \textbf{92.50} \\
    \texttt{place\_burger\_fries} & 96.00 & \textbf{99.00} & 97.50 & \textbf{99.00} & 97.00 & \textbf{98.00} \\
    \texttt{place\_can\_basket} & 71.00 & \textbf{69.00} & 70.00 & \textbf{81.00} & 65.00 & \textbf{73.00} \\
    \texttt{place\_cans\_plasticbox} & \textbf{99.00} & \textbf{96.00} & \textbf{97.50} & 98.00 & 94.00 & 96.00 \\
    \texttt{place\_container\_plate} & 96.00 & \textbf{100.00} & 98.00 & \textbf{99.00} & 99.00 & \textbf{99.00} \\
    \texttt{place\_dual\_shoes} & \textbf{94.00} & 88.00 & \textbf{91.00} & 89.00 & \textbf{91.00} & 90.00 \\
    \texttt{place\_empty\_cup} & \textbf{100.00} & \textbf{100.00} & \textbf{100.00} & 99.00 & \textbf{100.00} & 99.50 \\
    \texttt{place\_fan} & \textbf{96.00} & \textbf{96.00} & \textbf{96.00} & 94.00 & 94.00 & 94.00 \\
    \texttt{place\_mouse\_pad} & 83.00 & 89.00 & 86.00 & \textbf{95.00} & \textbf{90.00} & \textbf{92.50} \\
    \texttt{place\_object\_basket} & \textbf{89.00} & \textbf{88.00} & \textbf{88.50} & \textbf{89.00} & 85.00 & 87.00 \\
    \texttt{place\_object\_scale} & 90.00 & \textbf{97.00} & 93.50 & \textbf{93.00} & 96.00 & \textbf{94.50} \\
    \texttt{place\_object\_stand} & 90.00 & \textbf{94.00} & 92.00 & \textbf{93.00} & \textbf{94.00} & \textbf{93.50} \\
    \texttt{place\_phone\_stand} & 97.00 & \textbf{99.00} & 98.00 & \textbf{99.00} & 98.00 & \textbf{98.50} \\
    \texttt{place\_shoe} & \textbf{96.00} & \textbf{99.00} & \textbf{97.50} & 94.00 & \textbf{99.00} & 96.50 \\
    \texttt{press\_stapler} & 90.00 & \textbf{97.00} & 93.50 & \textbf{93.00} & 95.00 & \textbf{94.00} \\
    \texttt{put\_bottles\_dustbin} & \textbf{95.00} & 90.00 & \textbf{92.50} & 89.00 & \textbf{95.00} & 92.00 \\
    \texttt{put\_object\_cabinet} & \textbf{94.00} & \textbf{89.00} & \textbf{91.50} & 88.00 & \textbf{89.00} & 88.50 \\
    \texttt{rotate\_qrcode} & 93.00 & 89.00 & 91.00 & \textbf{95.00} & \textbf{90.00} & \textbf{92.50} \\
    \texttt{scan\_object} & 89.00 & 92.00 & 90.50 & \textbf{93.00} & \textbf{94.00} & \textbf{93.50} \\
    \texttt{shake\_bottle} & \textbf{100.00} & \textbf{100.00} & \textbf{100.00} & \textbf{100.00} & \textbf{100.00} & \textbf{100.00} \\
    \texttt{shake\_bottle\_horizontally} & \textbf{100.00} & \textbf{100.00} & \textbf{100.00} & \textbf{100.00} & 99.00 & 99.50 \\
    \texttt{stack\_blocks\_three} & \textbf{95.00} & \textbf{97.00} & \textbf{96.00} & 77.00 & 70.00 & 73.50 \\
    \texttt{stack\_blocks\_two} & \textbf{100.00} & \textbf{100.00} & \textbf{100.00} & \textbf{100.00} & \textbf{100.00} & \textbf{100.00} \\
    \texttt{stack\_bowls\_three} & 80.00 & 81.00 & 80.50 & \textbf{87.00} & \textbf{82.00} & \textbf{84.50} \\
    \texttt{stack\_bowls\_two} & 92.00 & \textbf{98.00} & \textbf{95.00} & \textbf{93.00} & 96.00 & 94.50 \\
    \texttt{stamp\_seal} & \textbf{90.00} & \textbf{94.00} & \textbf{92.00} & 87.00 & 89.00 & 88.00 \\
    \texttt{turn\_switch} & 61.00 & 59.00 & 60.00 & \textbf{70.00} & \textbf{78.00} & \textbf{74.00} \\
    \midrule
    \textbf{Average} & 91.88 & 91.78 & 91.83 & \textbf{93.18} & \textbf{92.34} & \textbf{92.76} \\

  \end{longtable}
\endgroup

% \clearpage
\section{Training Data, Data Collection, and Robot Platforms}
\label{app:data_collection_platforms}

% Use LaTeX boxes until the actual images are supplied; no robot photos are synthesized.
\newcommand{\appimageplaceholder}[2]{%
    \fbox{\begin{minipage}[c][#1][c]{\dimexpr\linewidth-2\fboxsep-2\fboxrule\relax}
        \centering\color{gray}\textbf{#2}\par\smallskip
        \small Image to be added
    \end{minipage}}}

\subsection{Training Data Overview}
\label{app:training_data_overview}
Figure~\ref{fig:app_training_data_overview} illustrates ten real-world manipulation tasks on the Spirit AI MOZ1 platform. We use a VR device to collect these real-world data.

\subsection{Robot Platforms}
\label{app:robot_platforms}
Figure~\ref{fig:app_robot_platforms} shows the Spirit AI MOZ1 and ROKAE AR5-5\_0.7 platforms used in our real-world experiments, together with the deployment pipeline. The green, red, and orange markers indicate the head-mounted camera, wrist-mounted RealSense cameras, and grippers, respectively. The camera views provide observations of the workspace and the regions near the grippers. During deployment, camera images and a natural-language task instruction are passed to the model running on a host computer, which predicts actions and sends them to the robot for execution.

\begin{figure}[t]
    \centering
    \includegraphics[width=\linewidth,height=0.7\textheight,keepaspectratio]{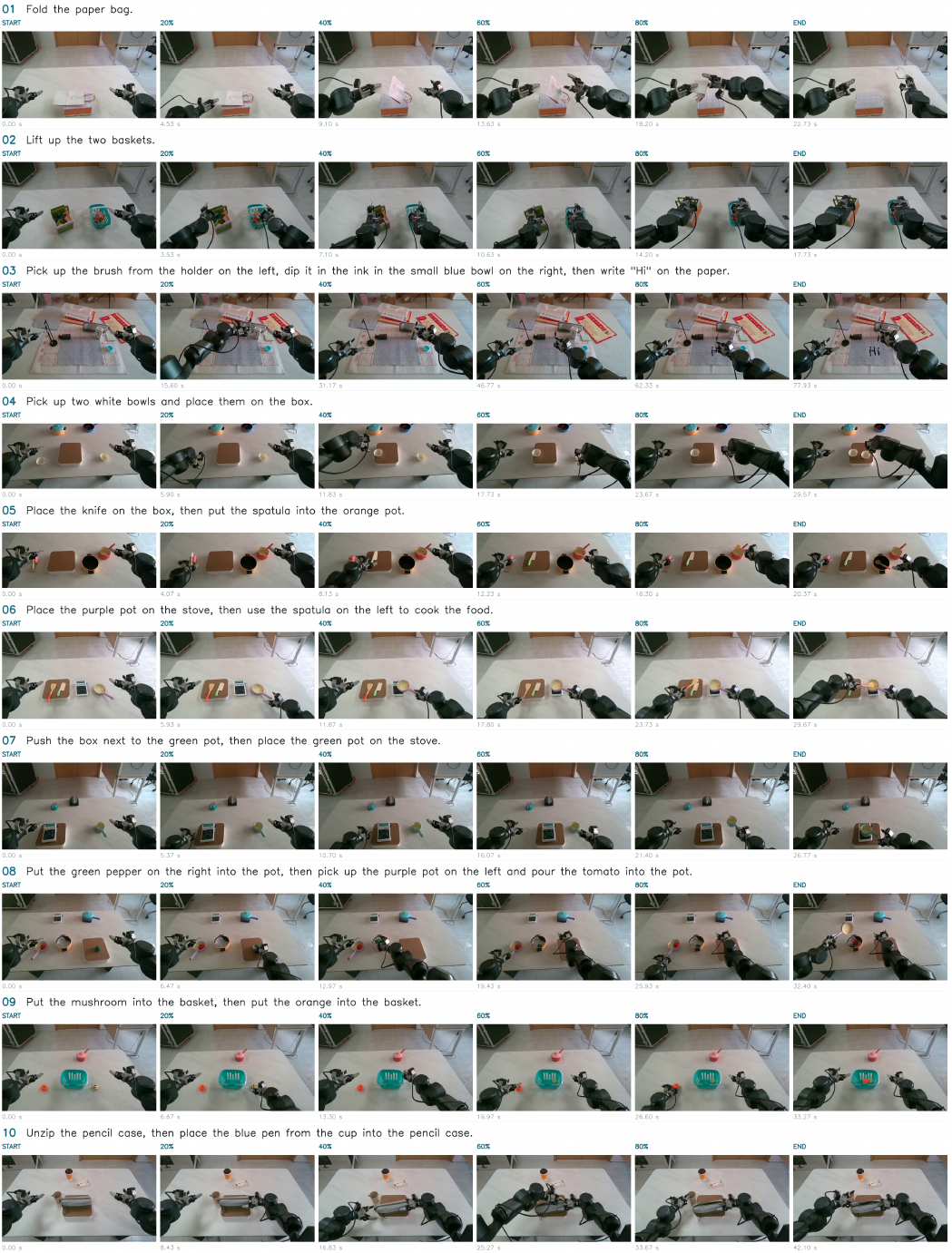}
    \caption{Overview of ten real-world manipulation tasks on Spirit AI MOZ1. Each row shows six frames from the high camera, covering an episode from start to finish.}
    \label{fig:app_training_data_overview}
\end{figure}

\begin{figure}[!t]
    \centering
    \includegraphics[width=\linewidth]{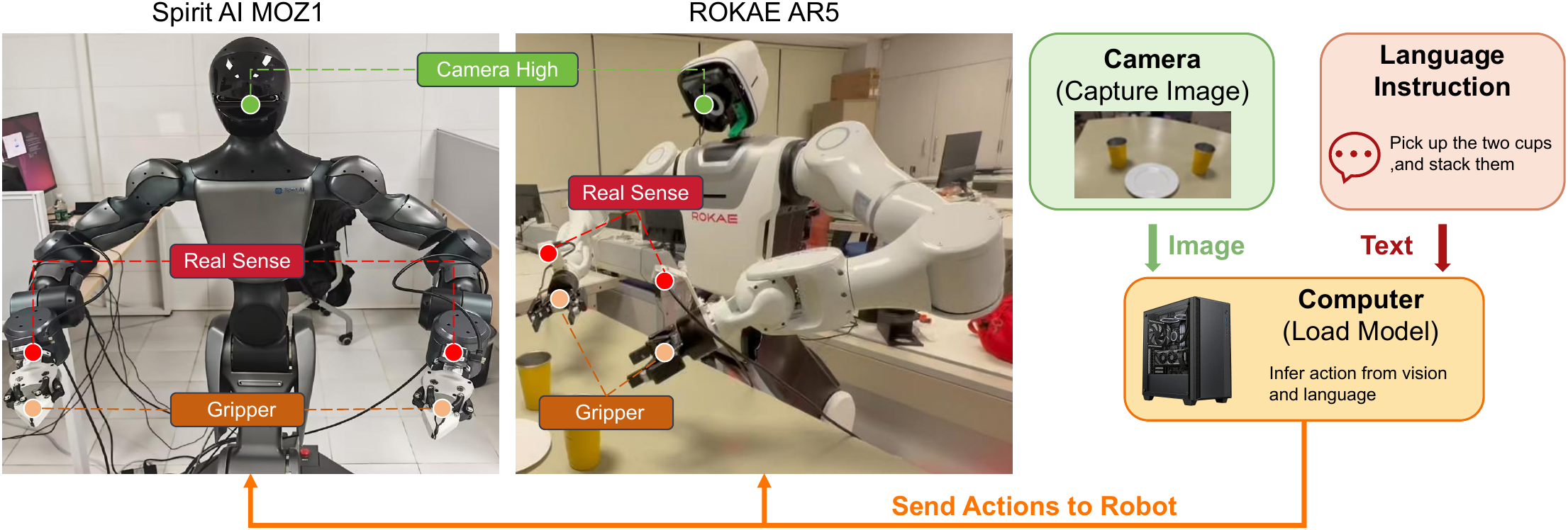}
    \caption{Robot platforms and deployment pipeline. Left and center: Spirit AI MOZ1 and ROKAE AR5-5\_0.7 with camera and gripper annotations. Right: camera images and language instructions are processed by the model on a host computer, and predicted actions are sent to the robot.}
    \label{fig:app_robot_platforms}
\end{figure}

\section{Analysis of Shared Action Experience}
\label{app:visualization_real_world}

\subsection{Cross-Task Similarity of Action Content Embeddings}
\label{app:aed_shared_experience}
 Manipulation tasks with distinct goals and visual contexts nevertheless share reusable action experience. For example, placing a can into a basket can draw on grasping, transporting, and releasing behaviors also used in other pick-and-place tasks. We examine this underlying relationship among manipulation tasks from two complementary perspectives: the similarity structure of action embeddings and the reuse of Action Experience Dictionary (AED) entries across task windows.

\begin{figure}[!htb]
    \centering
    \includegraphics[width=\linewidth]{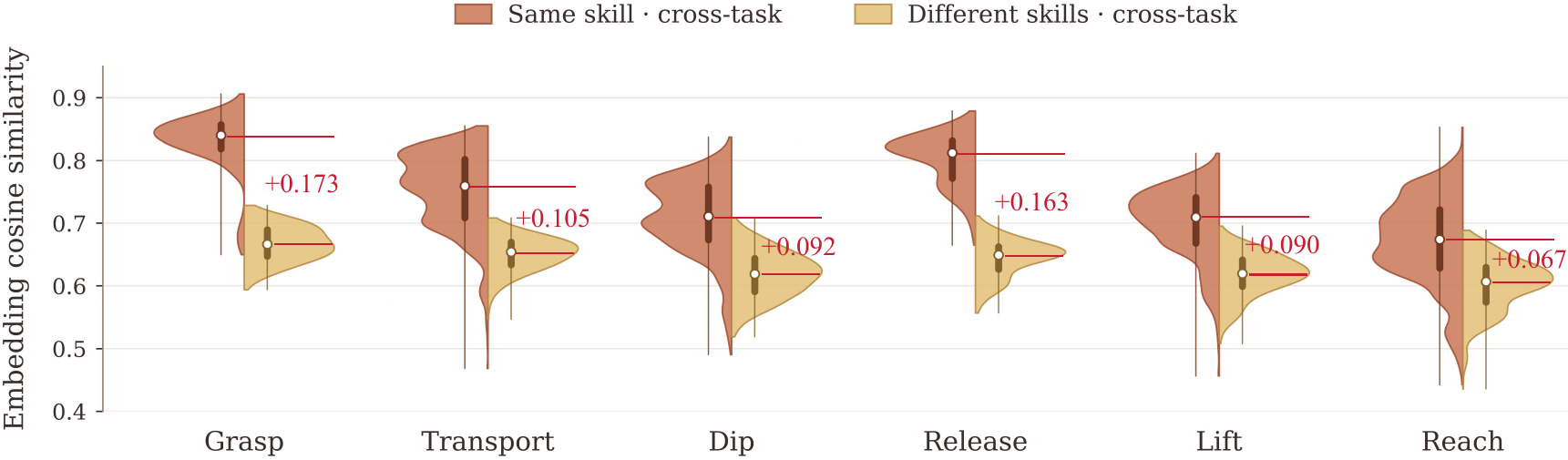}
    \caption{\textbf{Cross-task similarity of action content embeddings.} Orange and yellow compare the same skill and different skills across tasks, respectively. Each distribution sample is a weighted average cosine-similarity score for a cross-task pair; white points mark the medians of these task-pair scores. The annotated gaps for grasp, transport, dip, release, lift, and reach are $+0.173$, $+0.105$, $+0.092$, $+0.163$, $+0.090$, and $+0.067$, respectively.}
    \label{fig:app_aed_distribution}
\end{figure}

\paragraph{Representation and statistical unit.}
For each four-action interval, its FAST+ tokens undergo AED lookup and projection, followed by mean pooling over the retrieved valid tokens to produce a 1024-dimensional \emph{action content embedding}. This is the aggregated action representation before adding positional encoding or applying visual cross-attention, rather than a single AED vocabulary vector or a visually conditioned action embedding. Each sample in the violin distributions is a weighted average of cosine similarities for a cross-task pair, not an individual action vector or a single vector-pair similarity. The white points summarize these task-pair scores by their medians.

\paragraph{Cross-task organization of action content.}
Figure~\ref{fig:app_aed_distribution} shows higher median task-pair similarity for the same skill than for different skills across all six categories. Both comparison groups are cross-task. The largest annotated gaps occur for grasping ($0.173$) and releasing ($0.163$), while transporting also exhibits a positive gap ($0.105$). These behaviors directly match the reusable action experience in our motivation; dipping, lifting, and reaching extend the pattern beyond pick-and-place.

The overlapping distributions indicate graded relationships among skills rather than completely disjoint categories. Because this analysis precedes positional encoding and visual cross-attention, it locates the observed cross-task structure in the pooled action content itself. This supports the intended role of historical action trajectories as a source of reusable action experience across tasks.

\paragraph{Connection to the proposed mechanism.}
The pretrained action tokenizer supplies indices into the learnable AED. Repeated IDs establish access to the same dictionary entries, while the similarity distributions characterize the pooled, projected action content embeddings. An individual token ID need not denote an entire high-level skill. These complementary views connect shared dictionary access with the organization of action experience across tasks.

% \clearpage
\subsection{Shared AED Entries across Manipulation Tasks}
\label{app:aed_shared_entries}

\paragraph{Retrieval frequency.}
For each selected task window, the eight original four-action intervals are paired into four eight-action statistical groups. This grouping retains the original tokenization; it does not re-tokenize eight-action segments. Let $\mathcal{G}_{t,g}$ be the set of token IDs occurring in group $g$ of the selected window for task $t$. Then
\begin{equation}
    \operatorname{Frequency}(t,k)
    = \frac{1}{4}\sum_{g=1}^{4}\mathbf{1}\{k\in\mathcal{G}_{t,g}\}\times100\%.
    \label{eq:app_group_coverage}
\end{equation}
Thus, $75\%$ means presence in three of the four groups, irrespective of repetitions within a group. Retrieval is restricted to the selected window and is not a whole-task token frequency.

\paragraph{Shared dictionary access.}
In Fig.~\ref{fig:shared_tokens}, ID 300 occurs in all four task windows, with frequency of $75\%$, $100\%$, $25\%$, and $100\%$ for T1--T4. ID 309 has frequency of $75\%$ and $100\%$ in the two pick-and-place windows (T1 and T2); ID 308 has frequency of $75\%$, $100\%$, and $100\%$ in T1, T3, and T4. These overlapping, non-identical sets reveal shared AED access across different objects and goals, spanning both the pick-and-place examples in our motivation and drawer and stove interactions.

\begin{figure}[!t]
    \centering
    \includegraphics[width=0.98\linewidth]{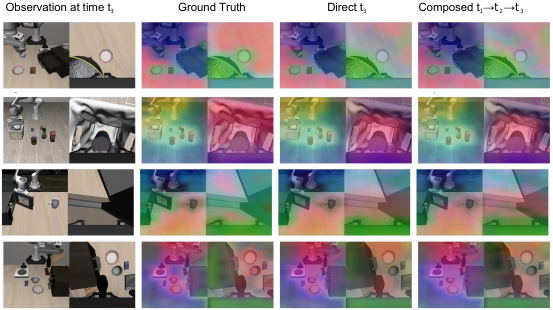}
    \caption{\textbf{Additional transition-prediction examples.} Each row shows, from left to right, the observation at $t_3$, a PCA visualization of its ground-truth features, the endpoint features estimated by direct $t_1\rightarrow t_3$ prediction, and those estimated by composing $t_1\rightarrow t_2$ and $t_2\rightarrow t_3$ predictions. The two prediction routes produce similar spatial feature patterns across the illustrated scenes.}
    \label{fig:app_transition_examples}
\end{figure}

As emphasized in the Introduction, similar motions can serve different purposes depending on objects and their spatial relationships. Shared AED entries supply reusable action experience, while subsequent visual conditioning supplies interaction context and action intent for the target task. These results indicate combining common action content with task-relevant visual information to exploit underlying relationships among manipulation tasks.

\section{Additional Qualitative Results}
\label{app:qualitative_results}

\subsection{Transition Prediction}
\label{app:transition_prediction}

Figure~\ref{fig:app_transition_examples} examines whether the learned transitions describe visual changes consistently across temporal intervals. For three time steps $t_1<t_2<t_3$, the direct prediction adds the predicted change over $t_1\rightarrow t_3$ to the features at $t_1$, whereas the composed prediction adds the changes over $t_1\rightarrow t_2$ and $t_2\rightarrow t_3$. Both therefore estimate the same endpoint features at $t_3$. The ground-truth column visualizes features extracted from the observation at $t_3$.

Across the four examples, the direct and composed predictions exhibit similar spatial organization and broadly reproduce the ground-truth feature patterns around the robot, scene objects, and surrounding surfaces. Agreement with the ground truth indicates that the two routes capture meaningful endpoint structure, while agreement between the routes supports temporal composition consistency: subdividing an interval yields a compatible estimate of the resulting visual state. This observation is consistent with the main-paper analysis linking the motion-aware transition (MT) objective to composition error, and supports learning action-related visual transitions across temporal scales.

\subsection{Visual Conditioning}
\label{app:visual_conditioning_examples}

Figure~\ref{fig:app_visual_conditioning} illustrates how historical action embeddings are associated with visual regions over a manipulation sequence. The eight panels correspond to $\mathbf{f}^{\star}[0]$ through $\mathbf{f}^{\star}[7]$ in temporal order, with two camera views in each panel. The overlaid responses show the visual regions associated with each action embedding, and the corresponding visualized trajectories are also provided.

\begin{figure}[t]
    \centering
    \includegraphics[width=0.98\linewidth]{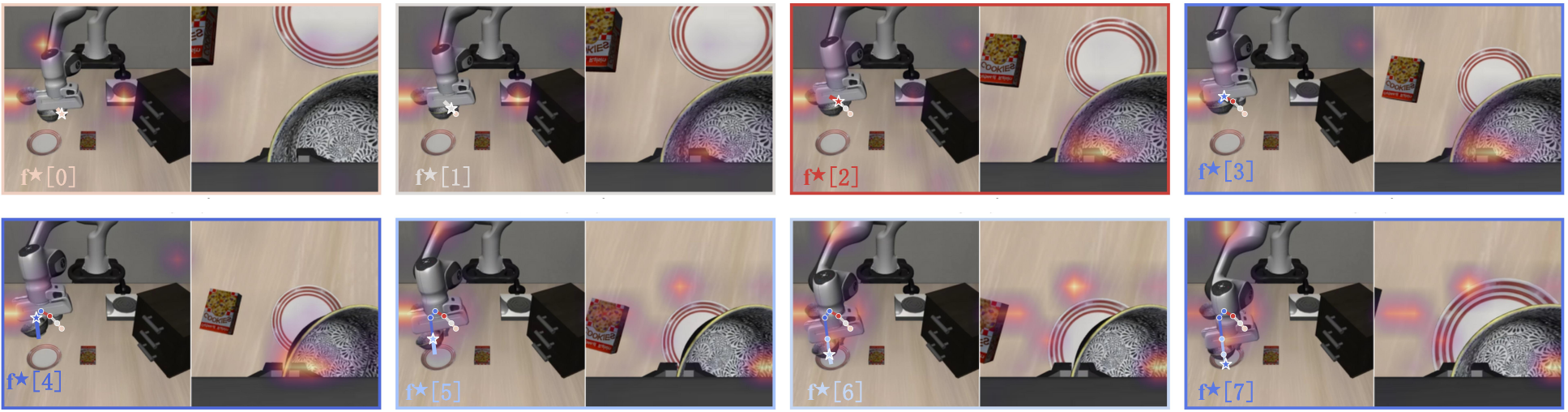}
    \caption{\textbf{Visual context associated with historical action embeddings.} Panels are ordered from left to right and top to bottom, corresponding to $\mathbf{f}^{\star}[0]$--$\mathbf{f}^{\star}[7]$. Each panel pairs two camera views with overlaid visual responses and trajectory markers. The responses vary with the interaction stage and include regions around the robot gripper and nearby objects.}
    \vspace{-6mm}
    \label{fig:app_visual_conditioning}
\end{figure}

The response patterns change as the gripper moves through the scene. In the earlier panels, visible responses occur near the robot and objects adjacent to the gripper; in later panels, responses also appear around the gripper--object interaction and the plate region as the trajectory approaches it. This stage-dependent association is consistent with visual conditioning supplying the object and spatial context needed to interpret historical motions. Such context matters because a reusable motion alone does not specify which object it acts on or how that object relates to the current goal.

Together with the shared action-content analysis in Appendix~\ref{app:aed_shared_experience}, these examples support complementary roles for the two components: AED provides reusable action experience, and visual conditioning associates that experience with the observed interaction context. The localized responses are also consistent with the intended emphasis of MT supervision on action-relevant visual information.

\section{Ablation Studies}
\label{app:ablations}

\subsection{Experimental Protocol}
\label{app:ablation_protocol}
The ablation uses seed 0, 32 predicted actions per call, ten executed actions before replanning, with 50 trials per task across ten LIBERO-10 tasks. SR denotes success rate in percent. All six ablation plots use 97.4\% as the shared full-model reference.

\subsection{Action Aggregation and Temporal Sampling}
\label{app:action_ablation_design}
\paragraph{Action aggregation.}
In Fig.~\ref{fig:ablation_aggregation}, Agg$\rightarrow$Tok aggregates each four-action interval before FAST+ tokenization, Tok$\rightarrow$Agg jointly tokenizes the four actions before pooling their embeddings, and MLP encodes the aggregated continuous actions with a multilayer perceptron. The Agg$\rightarrow$Tok (Ours) outperforms the other methods.

\paragraph{Temporal sampling.}
Figure~\ref{fig:ablation_sampling} compares Start$\rightarrow$Span, which samples a start first and then a valid span, Span$\rightarrow$Start, which reverses this order, a fixed temporal window, and the configuration without interval sampling. The 96.8\% result jointly removes random sampling and the motion-aware transition (MT) objective.

\begin{figure}[!htbp]
\centering
\captionsetup{position=bottom,skip=3pt,font=small}
\begin{minipage}[t]{0.485\linewidth}
\centering
\includegraphics[width=\linewidth]{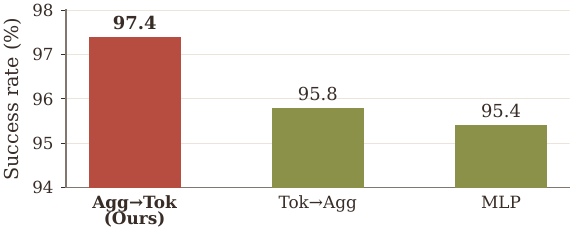}
\caption{Action aggregation.}\label{fig:ablation_aggregation}
\end{minipage}
\hfill
\begin{minipage}[t]{0.485\linewidth}
\centering
\includegraphics[width=\linewidth]{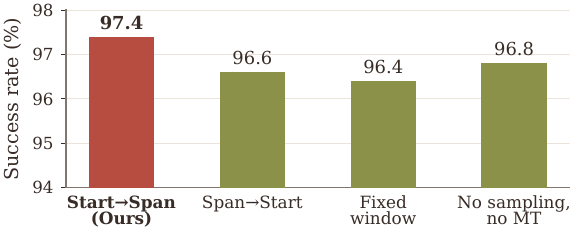}
\caption{Temporal sampling.}\label{fig:ablation_sampling}
\end{minipage}
\end{figure}

\subsection{Sampling Window Distribution}
\label{app:window_sampling_distribution}
To clarify the Start$\rightarrow$Span strategy, Fig.~\ref{fig:app_window_sampling} visualizes the probability of sampling each temporal window for motion-aware transition supervision. Let $r=t'-t$ denote the start offset in visual intervals. The sampler first draws $r$ uniformly from $\{1,\ldots,V-1\}$, then draws $\Delta$ uniformly from $\{1,\ldots,V-r\}$. Consequently,
\begin{equation}
    P(r,\Delta)=\frac{1}{(V-1)(V-r)},
    \qquad 1\leq r\leq V-1,\quad 1\leq\Delta\leq V-r,
    \label{eq:app_window_joint}
\end{equation}
with zero probability outside this support. Summing over valid start offsets gives
\begin{equation}
    P(\Delta)=\frac{1}{V-1}\sum_{r=1}^{V-\Delta}\frac{1}{V-r},
    \qquad \Delta\in\{1,\ldots,V-1\}.
    \label{eq:app_window_marginal}
\end{equation}

\begin{figure}[!htb]
    \centering
    \includegraphics[width=\linewidth]{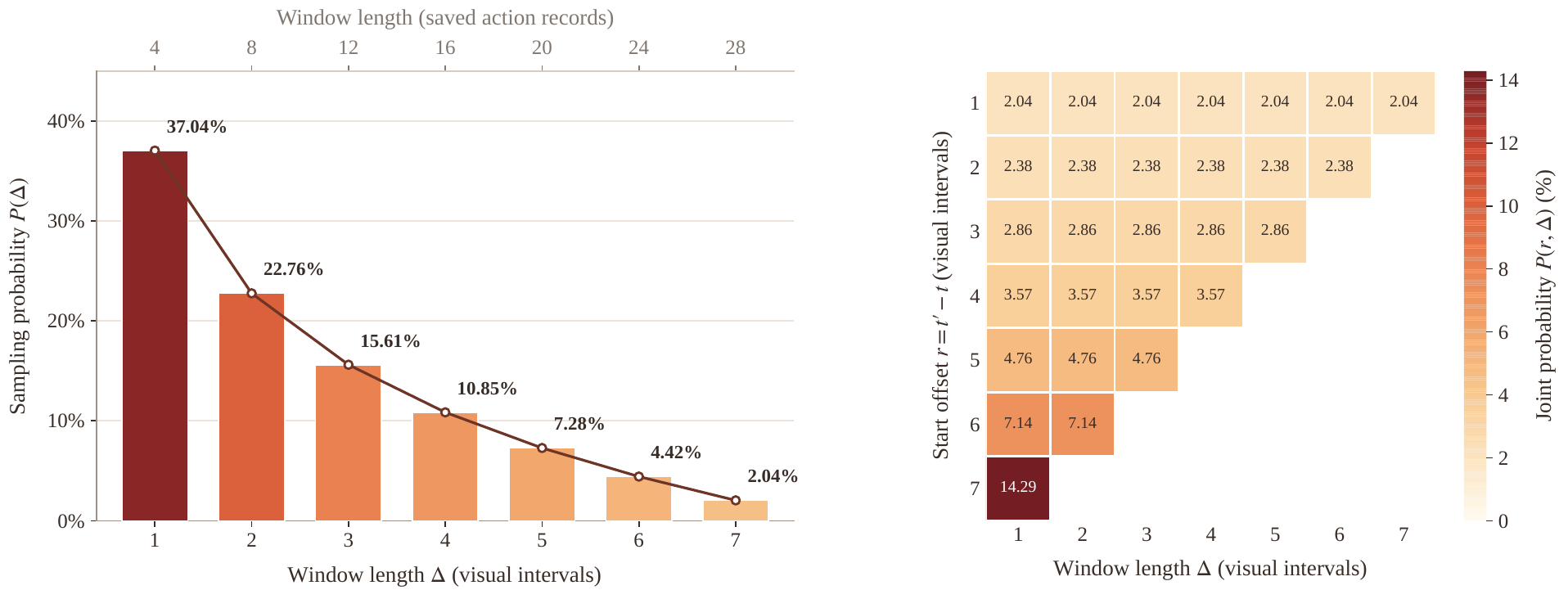}
    \caption{\textbf{Temporal-window sampling probabilities for Start$\rightarrow$Span}, with $V=8$. Left: marginal length distribution $P(\Delta)$; the upper axis counts saved action records (four per visual interval). Right: joint distribution $P(r,\Delta)$ over start offsets and lengths; blank cells are invalid windows. All values are percentages.}
    \label{fig:app_window_sampling}
\end{figure}

Uniform draws at each stage do not give uniform window lengths: for $V=8$, $P(\Delta=1)=37.04\%$ and $P(\Delta=7)=2.04\%$. Short windows are valid at more start offsets, including late starts with few remaining choices. The resulting MT supervision emphasizes short-term visual transitions while retaining positive probability for every valid longer interval.

\subsection{Visual Encoder and Query Count}
\label{app:visual_ablation_design}
\paragraph{Pretrained visual encoder.}
Figure~\ref{fig:ablation_visual_encoder} compares LingBot-Vision and DINOv3 as alternative pretrained encoders for visual feature extraction.

\paragraph{Visual-memory query count.}
Figure~\ref{fig:ablation_queries} varies the number of learnable visual-memory queries across 32, 64, 128, and 256, with 128 as the proposed setting.

\begin{figure}[!htbp]
\centering
\captionsetup{position=bottom,skip=3pt,font=small}
\begin{minipage}[t]{0.485\linewidth}
\centering
\includegraphics[width=\linewidth]{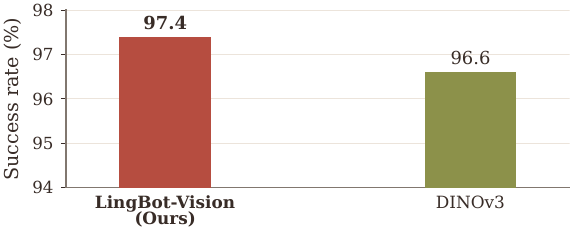}
\caption{Pretrained visual encoder.}\label{fig:ablation_visual_encoder}
\end{minipage}
\hfill
\begin{minipage}[t]{0.485\linewidth}
\centering
\includegraphics[width=\linewidth]{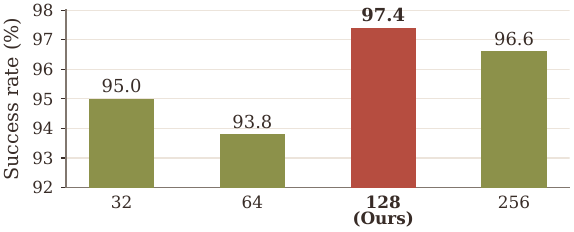}
\caption{Visual-memory query count.}\label{fig:ablation_queries}
\end{minipage}
\end{figure}

\subsection{Prefix Content and Injection}
\label{app:prefix_ablation_design}
\paragraph{Prefix content.}
In Fig.~\ref{fig:ablation_prefix_content}, Vision-only constructs the history prefix from historical visual features, whereas Action-only uses historical action embeddings without VC. The 95.6\% Action-only setting also removes MT; the 94.8\% setting jointly removes AED, VC, and MT.

\paragraph{Prefix injection.}
In Fig.~\ref{fig:ablation_prefix_injection}, Prefix prepends the history tokens to the action sequence, whereas Cross-attn conditions the action expert through separate residual cross-attention modules.

\begin{figure}[!htbp]
\centering
\captionsetup{position=bottom,skip=3pt,font=small}
\begin{minipage}[t]{0.485\linewidth}
\centering
\includegraphics[width=\linewidth]{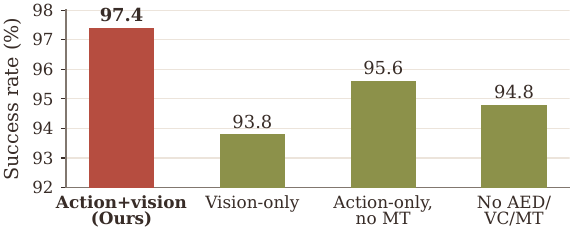}
\caption{Prefix content.}\label{fig:ablation_prefix_content}
\end{minipage}
\hfill
\begin{minipage}[t]{0.485\linewidth}
\centering
\includegraphics[width=\linewidth]{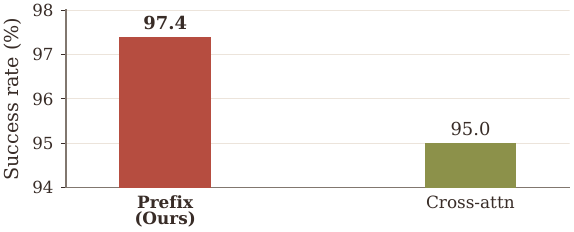}
\caption{Prefix injection.}\label{fig:ablation_prefix_injection}
\end{minipage}
\end{figure}
\end{document}